\documentclass{article}
\usepackage{arxiv}

\usepackage[utf8]{inputenc} 
\usepackage[T1]{fontenc}    
\usepackage{hyperref}       
\usepackage{url}            
\usepackage{booktabs}       
\usepackage{amsfonts}       
\usepackage{nicefrac}       
\usepackage{microtype}      
\usepackage{xcolor}         
\usepackage{mathtools}

\usepackage{latexsym,amsmath,amssymb,bm}
\usepackage[font=footnotesize, skip=5pt]{caption}
\usepackage{times}
\usepackage{amsthm}
\usepackage[retainorgcmds]{IEEEtrantools}
\usepackage{mdwlist}
\usepackage{ctable}
\usepackage{enumitem}
\usepackage{listings}
\usepackage{thmtools,thm-restate}

\newtheorem{defn}{Definition}

\newtheorem{lma}{Lemma}
\newtheorem{assm}{Assumption}

\DeclareMathOperator*{\argmin}{arg\,min}
\DeclareSymbolFont{bbold}{U}{bbold}{m}{n}
\DeclareSymbolFontAlphabet{\mathbbold}{bbold}

\title{Towards an Understanding of Benign Overfitting in Neural Networks}

\author{%
  Zhu Li\\
  Gatsby Computational Neuroscience Unit\\
  University College London\\
  London, UK, W1T 4JG \\
  \texttt{zhu.li@ucl.ac.uk} \\
  \And
  Zhi-Hua Zhou \\
  National Key Laboratory for Novel Software Technology \\
  Nanjing University\\
  Nanjing, China, 210023 \\
  \texttt{zhouzh@lamda.nju.edu.cn} \\
  \And
  Arthur Gretton\\
  Gatsby Computational Neuroscience Unit\\
  University College London\\
  London, UK, W1T 4JG\\
  \texttt{arthur.gretton@gmail.com} \\
  
}

\begin{document}

\maketitle
\vspace{-1em}
\begin{abstract}
Modern machine learning models often employ a huge number of parameters and are typically optimized to have zero training loss; yet surprisingly, they possess near-optimal prediction performance, contradicting classical learning theory. We examine how these benign overfitting phenomena occur in a two-layer neural network setting where sample covariates are corrupted with noise. We address the high dimensional regime, where the data dimension $d$ grows with the number $n$ of data points. Our analysis combines an upper bound on the bias with matching upper and lower bounds on the variance of the interpolator (an estimator that interpolates the data). These results indicate that the excess learning risk of the interpolator decays under mild conditions. We further show that it is possible for the two-layer ReLU network interpolator to achieve a near minimax-optimal learning rate, which to our knowledge is the first generalization result for such networks. Finally, our theory predicts that the excess learning risk starts to increase once the number of parameters $s$ grows beyond  $O(n^2)$, matching recent empirical findings.
\end{abstract}
\vspace{-.5em}
\section{Introduction}
In modern machine learning, overparameterized models have been applied to a wide range of learning tasks, such as natural language processing \citep{hinton2012deep} and computer vision \citep{he2015delving}. Despite their empirical success, the theoretical properties of these overparameterized models remain unclear. In particular, recent empirical studies \citep{zhang2016understanding,belkin2018understand,belkin2019reconciling} demonstrate that overparameterized models can not only perfectly fit the noisy training data, but also achieve near-optimal prediction accuracy. Furthermore, a plot of the prediction error as a function of the number of parameters reveals that the learning risk for overparameterized models displays a \textit{double descent} behavior \citep{belkin2019reconciling}.  

The phenomenon where overparameterized models can generalize well has recently been characterized as \textit{benign overfitting} \citep{bartlett2020benign} and has attracted significant research interest~\citep[see, e.g.,][]{bartlett2020benign,belkin2018understand,belkin2019reconciling,belkin2019two,mei2019generalization,liao2020random,hastie2019surprises,deng2019model,derezinski2019exact,zhou2021over,liang2020just,liang2020multiple}. Existing research has been focused on studying the learning risk behavior of the \textit{interpolator} (the estimator that interpolates the training data) under two regimes. The first regime investigates the asymptotic learning risk of the interpolator~\citep{mei2019generalization,hastie2019surprises,liao2020random,wu2020optimal,richards2021asymptotics}. These works derive the asymptotic learning risk by assuming that the data dimension $d$ and number of training samples $n$ (and number of parameters $s$ in the non-linear regression case) grow simultaneously while their ratio is kept fixed. These results indicate that the prediction error can exhibit a double descent curve or a multiple descent curve with respect to the ratio between $d,n$ and $s$. However, asymptotic results are often limited to linear regression or random feature learning, and require further restrictive assumptions such as a Gaussian kernel or a Gaussian data generating distribution. In addition, these results do not provide a finite sample analysis of the learning risk, and hence do not give insight into the learning rate of the interpolator.

To overcome the above limitations, the second line of research studies the finite sample behavior of the learning risk. In linear regression, \cite{bartlett2020benign} provide a first finite sample study of the interpolator. By assuming that the covariates belong to an infinite-dimensional Hilbert space, and that the data generating distribution follows a subgaussian distribution, they derive a matching upper and lower bound on the excess risk. These results indicate the condition for the interpolator to be consistent: the covariance operator spectrum has to decay slowly enough so that the sum of the tail of its eigenvalues is large compared to $n$. Following this, \cite{chinot2020benign} extend the analysis into the large deviation regime where the label noise exhibits heavy tail behavior and \cite{bunea2020interpolation} investigate the benign overfitting phenomenon for the widely used factor regression models. Their finite sample analysis reveals that when the response and the covariate are jointly low dimensional, the interpolator can obtain the optimal prediction accuracy. \cite{mahdaviyeh2019risk} further study the learning risk of the interpolator under the spike covariance model. They show that the excess learning risk of the interpolator can vanish as long as a fixed number of leading eigenvalues grow with $n$ and are significantly larger than the rest. \cite{muthukumar2020harmless} explore the overparameterized regime in linear regression with noisy data, and provide the lower bound for the learning risk of the interpolator. Their results show that the lower bound can decay to zero and overparameterization is necessary for harmless interpolation. In kernel learning, \cite{liang2020just} study the kernel interpolator and demonstrate that the prediction error converges provided that the kernel function has a nice curvature property. Generalization properties of high dimensional kernel ridge regression are investigated in \cite{liu2021kernel} in both underfitted and overfitted regime. Depending on the values of $d$ and $n$, the learning risk displays unimodal or monotonically decreasing behavior. Finally, in \cite{li2020benign}, the finite sample risk bound is studied in the noisy random feature setting. With no additional assumptions on the data generating distribution and the kernel structure, they derive the double descent learning curve under the setting where the features are corrupted with subgaussian noise.

While many explanations of benign overfitting have been proposed, existing results mainly focus on either linear or kernel regression models. In particular, we do not have a thorough understanding of the overparameterized neural network model, which is the setting where benign overfitting phenomenon was observed in the first place. Therefore, in this paper, we study the finite sample risk behavior for a two-layer ReLU neural network. We aim to paint a more comprehensive picture by studying the generalization properties of the interpolator under two assumptions: i) the high dimensional regime where $d = O(n^{\alpha}), \alpha \in (0,1)$; and ii) the noisy covariates regime where the sample covariates are corrupted by independent noise (see more details in Section \ref{sec:main_result}). Our results shed light on how benign overfitting occurs. Specifically, we make the following contributions:
\begin{itemize}
    \item Theorem \ref{theo:var_upper} provides an upper bound on the \textit{variance} of the interpolator. Our analysis reveals that the variance of the interpolator converges as long as the decay rate of the covariate noise is lower than that of the spectrum of the covariance matrix. Moreover, the convergence holds even in the heavily parameterized setting where the order of the number of parameters ($s$) can be taken up to $O(nd)$. In addition, Theorem \ref{theo:var_lower} provides a matching lower bound on the variance, which indicates that our upper bound is tight;
    
    \item We study the properties of the \textit{bias} term in Section \ref{sec:bias}, where Theorem \ref{theo:bias_upper} provides a finite sample convergence bound. Our results show that convergence of the bias depends on the interplay between the data dimension $d$ and the covariate noise decay rate. In particular, if the data lives in a relatively low dimension ($\alpha \leq 1/2$) or the covariate noise has fast decay, the convergence of the bias is guaranteed;
    
    \item Section \ref{sec:ex_risk} discusses the generalization properties of the interpolator for two-layer ReLU network. We show that the interpolator can achieve a near minimax optimal learning rate $O(\sqrt{\log n /n})$ under mild conditions on the data dimension and noise decay rate. Moreover, we demonstrate that the prediction error displays a multiple descent, instead of a double descent, behavior in the presence of covariate noise. Our findings are supported by empirical evidence \citep{d2020triple,nakkiran2020optimal,adlam2020neural}. 
\end{itemize}
\vspace{-.5em}
\subsection{Additional Related Work}
Besides regression, benign overfitting has also been studied in the linear classification setting \citep{cao2021risk,muthukumar2020classification,wang2020benign,chatterji2020finite}. Specifically, using the overparameterized linear model with Gaussian feature, \cite{muthukumar2020classification} demonstrate that the solution of the hard-margin support vector machine (SVM) is equivalent to the minimum norm interpolator under square loss. Utilizing the equivalence, they provide the first non-asymptotic risk bounds of the minimum norm interpolator for the classification task. The benign overfitting phenomenon in the subgaussian/Gaussian mixture models is also studied in \cite{cao2021risk,chatterji2020finite} and \cite{wang2020benign}. Through the equivalence result between classification and regression or the implicit bias of gradient descent for logistic regression, they provide population risk bounds for overparameterized classification models.

Our work is also related to recent studies on understanding the risk of the estimators in the high dimensional regime. For example, \cite{rakhlin2019consistency} derive the risk bound for the Laplace kernel interpolator and show that the risk does not converge unless the data dimension $d$ grows with $n$. \cite{liang2020just} prove that the risk of the interpolator in the regime where $d$ grows with $n$ can be upper bounded by a small quantity provided that the kernel exhibits certain favorable spectral properties. In a similar high dimensional setting, \cite{liang2020multiple} provide the risk of the kernel interpolator and demonstrate that the risk curve has a multiple descent shape. \cite{ghorbani2021linearized} study random feature and neural tangent kernel regression in the high dimensional setting. They show that the two estimators are equivalent in fitting certain order of polynomials and can achieve near-optimal prediction accuracy for vanishing ridge regularization.

\section{Background}
\subsection{Shallow ReLU Network and Regularized Empirical Risk Minimization}
Let $x$ and $y$ be random variables with joint probability distribution $\rho(x,y)  = \rho(x)\rho(y|x)$. Given training samples $\mathcal{D} = \{(x_i,y_i)\}_{i=1}^n$ drawn independently and identically distributed (i.i.d) from $\rho(x,y)$, let $X = [x_1, \dots,x_n]^T$ and $Y = [y_1,\cdots,y_n]^T$ denote the training covariates and outputs. We consider the following function class to perform learning \[\mathcal{H} =: \left\{\sum_{i=1}^n  \sigma(\mathbf{W}x_i)^T\beta, \beta = [\beta_1,\dots,\beta_s]^T \in \mathbb{R}^s\right\},\] where $\sigma(x) = \max\{0,x\}$ is the ReLU activation function applied entrywise. $\mathbf{W} \in \mathbb{R}^{s\times d}$ is a random matrix with each entry $\mathbf{W}_{ij}$ being i.i.d according to $\pmb{\mathsf{w}} \sim \mathcal{N}(0,\sigma_w^2)$, where a typical choice is $\sigma_w^2 = s^{-1}$. In other words, we are using a two-layer ReLU neural network to perform learning. Thus, given a testing data point $x'$, our prediction at this point is \[f(x',\mathbf{W},\beta) =  \beta^T\sigma(\mathbf{W}x').\]

To simplify our presentation, we let \[D_x(\mathbf{W}) = \text{diag}\left\{\mathbbold{1}_{ \{\mathbf{w}_1^Tx >0 \}}, \dots, \mathbbold{1}_{\{\mathbf{w}_s^Tx> 0\}} \right\},\] where $\mathbbold{1}_{A}$ is the indicator function for event $A$ and $\mathbf{w}_i \in \mathbb{R}^d$ is the $i$-th row of $\mathbf{W}$. We will also denote the feature vector and feature matrix as \[\mathbf{z}_{x}(\mathbf{W}) = \sigma(\mathbf{W}x) \in \mathbb{R}^{s},~~~\mathbf{Z}(\mathbf{W}) = [\mathbf{z}_{x_1}(\mathbf{W}),\dots,\mathbf{z}_{x_n}(\mathbf{W})]^T \in \mathbb{R}^{n\times s}.\] When the context is clear, we will write $D_x(\mathbf{W})$, $\mathbf{z}_{x}(\mathbf{W})$ and $\mathbf{Z}_X(\mathbf{W})$ as $D_x$, $\mathbf{z}_x$ and $\mathbf{Z}$ respectively. Thus, we have \[\mathbf{z}_{x} = D_x \mathbf{W}x, ~~~\mathbf{Z} = [D_{x_1}\mathbf{W}x_1, \dots, D_{x_n}\mathbf{W}x_n]^T,~~~f(x',\mathbf{W},\beta) = \beta^T\mathbf{z}_{x'}.\]

Throughout the paper, we consider the regression problem with $x \in \mathbb{R}^d, y \in \mathbb{R}$, and the squared loss $l(y,f(x))=(y-f(x))^2$. Under this setting, we formulate the regularized empirical risk minimization (ERM) learning as
\begin{IEEEeqnarray}{rCl}
\hat{f}^{\lambda} := \argmin_{f\in\mathcal{H}}\ \frac{1}{n}\sum_{i=1}^n (y_i-f(x_i))^2 + \lambda \Omega(f),\nonumber
\end{IEEEeqnarray}
where $\lambda$ is the hyperparameter and $\Omega(f)$ is some measure of the function complexity.

\subsection{Interpolator and the Excess Learning risk} \label{sec:regression}
As discussed before, we will be interested in the estimator that can fit the data perfectly. Since there are infinitely many of such estimators, we focus on the minimum norm estimator defined below:

\begin{defn}(Minimum Norm Estimator) \label{def:min_est}
Given covariates $X$ and response variables $Y$, we define the minimum norm least square (MNLS) estimator for one-hidden layer ReLU neural network as
\begin{IEEEeqnarray}{rCl}
\min_{\beta \in \mathbb{R}^s} \|\beta\|^2,~~~ \textnormal{such~that:~} \|\mathbf{Z}\beta-Y\|^2 = \min_{\theta}\|\mathbf{Z}\theta-Y\|^2. \nonumber 
\end{IEEEeqnarray}   
\end{defn}
By the projection theorem, it is easy to see that the closed form solution of the MNLS estimator is 
\begin{IEEEeqnarray}{rCl}
\tilde{\beta} = \mathbf{Z}(\mathbf{Z}\mathbf{Z}^T)^{\dagger}Y = (\mathbf{Z}^T\mathbf{Z})^{\dagger}\mathbf{Z}^TY, \label{eqn:mnls}
\end{IEEEeqnarray}  
where $A^{\dagger}$ denotes the pseudoinverse for matrix $A$.

To investigate the generalization property of the MNLS estimator, we rely on the notion of the excess learning risk. In the regularized ERM with squared loss function, the optimal estimating regression function at a point $x_0$ is given by \[f_*(x_0) = \mathbb{E}(y|x = x_0).\] Given the function $\hat{f}$ estimated based on $(X,Y)$, we define the excess learning risk as in~\cite{caponnetto2007optimal}
\begin{equation}
    \mathcal{E}(\hat{f}) = \mathbb{E}_{x,y}\left[(\hat{f}(x)-y)^2-  (f_*(x)-y)^2\right]. \nonumber
\end{equation}

Since $\hat{f}$ is estimated from $(X,Y)$, studying the excess risk $ \mathcal{E}(\hat{f}) $ requires us to understand the distribution of the label noise, which we denote as \[\epsilon = y - f_*(x),~~~\pmb{\epsilon} = [y_1 -f_*(x_1),\dots,y_n-f_*(x_n)]^T.\] However, knowing the exact distribution of $\epsilon$ is not feasible in practice. As a result, we will study the average risk instead to avoid making restrictive assumption on $\epsilon$: \[R(\hat{f})= \mathbb{E}_{\pmb{\epsilon}}(\mathcal{E}(\hat{f}) ).\]

\subsection{Bias-Variance Decomposition}
Our aim is to analyze the excess learning risk for the MNLS estimator, which starts with the bias-variance decomposition. We present the decomposition and introduce some relevant notation here to ease our subsequent discussion. 

Given a feature matrix $\mathbf{Z}$, we denote the Hilbert space it spans as $\mathcal{H}$, and the best estimator in $\mathcal{H}$ as \[f_{\mathcal{H}} := \argmin_{f \in \mathcal{H}} \mathbb{E}_{\rho}(f(x)-y)^2.\] Since $f_{\mathcal{H}} \in \mathcal{H}$, we let $f_{\mathcal{H}} (x) = \mathbf{z}_x^T \beta_{\mathcal{H}}$ for some $\beta_{\mathcal{H}} \in \mathbb{R}^s$. Given an estimator $\tilde{f}$, we define \[\tilde{f}(X) = [\tilde{f}(x_1),\cdots, \tilde{f}(x_n)]^T.\] Recall $\pmb{\epsilon} = [y_1-f_*(x_1),\dots,y_n-f_*(x_n)]^T$, we define the following two terms
\[\mathbf{B}_R  =  \mathbb{E}_{x}\left[\left(\mathbb{E}_{\pmb{\epsilon}} [\tilde{f}(x)] -f_{\mathcal{H}}(x)\right)^2\right], 
~~~~~~\mathbf{V}_R =  \mathbb{E}_{x} \textnormal{Var}_{\pmb{\epsilon}}( \tilde{f}(x)).\]
 
In the rest of the manuscript, we refer to $\mathbf{B}_R$ as the (squared) \textit{Bias} and $\mathbf{V}_R$ as the \textit{Variance} of the estimator $\tilde{f}$. Their relationship to the excess learning risk is described in the following lemma (proof in Appendix \ref{sec:b-v}).
\begin{restatable}{lma}{BiaVar} \label{lma:bia_var}
Let $\tilde{\beta}$ be the MNLS estimator from Eq.(\ref{eqn:mnls}) associated with feature matrix $\mathbf{Z}$, and $\Pi = \left(\mathbf{Z}^T\mathbf{Z}\right)^{\dagger}\mathbf{Z}^T\mathbf{Z} -I_s$. Assuming that $f^* \notin \mathcal{H}$, the bias and variance can be written as
\begin{IEEEeqnarray}{rCl} 
\mathbf{B}_R & = & \mathbb{E}_{x}\left\|\mathbf{z}_{x}^T\Pi\beta_{\mathcal{H}}\right\|^2, \nonumber\\
\mathbf{V}_R &= & \mathbb{E}_{x}\left\{\mathbb{E}_{\pmb{\epsilon}}\left\| \mathbf{z}_{x}^T(\mathbf{Z}^T\mathbf{Z})^{\dagger}\mathbf{Z}^T(Y -f_{\mathcal{H}}(X))\right\|^2\right\}. \nonumber
\end{IEEEeqnarray}
In addition, the misspecification error is defined as
\[\mathbf{M}_R := \mathbb{E}_{x} \left\{\mathbf{z}_{x}^T(\mathbf{Z}^T\mathbf{Z})^{\dagger}\mathbf{Z}^T \left(f_*(X) - f_{\mathcal{H}}(X) \right) \right\}^2 + \mathbb{E}_{x} \left(f_*(x) - f_{\mathcal{H}}(x)\right)^2, \]
then the following decomposition of the excess learning risk of $\tilde{\beta}$ holds
\begin{IEEEeqnarray}{rCl} 
R(\tilde{\beta}) \leq   3 (\mathbf{M}_R + \mathbf{B}_{R} + \mathbf{V}_{R}). \nonumber
\end{IEEEeqnarray}
\end{restatable}

Lemma \ref{lma:bia_var} states that the excess learning risk can be decomposed into the misspecification error ($\mathbf{M}_R$), the bias ($\mathbf{B}_R$) and the variance ($\mathbf{V}_R$). Classical learning theory \cite[Chapter 2.9]{friedman2001elements} indicates that when the model is relatively simple, $\mathbf{M}_R, \mathbf{B}_{R}$ are large but $\mathbf{V}_{R}$ is small. As the model complexity increases, $\mathbf{M}_R, \mathbf{B}_{R}$ decrease while $\mathbf{V}_{R}$ increases. This forms the famous U-shape learning curve. However, recent advancements in deep learning models \citep{zhang2016understanding,belkin2018understand,belkin2019reconciling} demonstrate that heavily overparameterized models or even interpolated models can still generalize well. Hence, in this paper, we are interested in the generalization properties of the ReLU neural networks in the overparameterized setting, and aim to provide the conditions under which benign overfitting occurs. Note that in the overparameterized regime, $\mathbf{M}_R$ is likely to be much smaller than $\mathbf{B}_R$ and $\mathbf{V}_R$. As a result, we will mainly focus on analyzing $\mathbf{B}_R$ and $\mathbf{V}_R$ by assuming $f_* = f_{\mathcal{H}}$, implying $\beta_* = \beta_{\mathcal{H}}$ and $\mathbf{M}_R =0$.

\section{Main Results}\label{sec:main_result}
We provide the main results concerning the generalization properties of overparameterized ReLU networks in this section. We will first list our assumptions below.

\begin{assm}\label{assm:regre}
We assume that both $x$ and $y$ have zero mean and $y = f_*(x) + \epsilon$ with $\mathbb{E}(\epsilon) = 0$ and $\text{Var}(\epsilon) = \sigma_0^2$.
\end{assm}
Assumption \ref{assm:regre} is typical in the regression setting where both $x$ and $y$ are standardized. Our next two assumptions concern the dimension and distribution of the covariate $x$.

\begin{assm}\label{assm:dimen}
Assume that the covariate is in the high-dimensional setting: $x \in \mathbb{R}^{d}$ and $d=O(n^{\alpha}), \alpha \in (0,1)$.
\end{assm}
Recent research interest has focused on analyzing the MNLS estimator in the high dimensional setting where $d$ grows with $n$~\citep[see, e.g.,][]{bartlett2020benign,hastie2019surprises,liang2020just,rakhlin2019consistency,liang2020multiple}. In particular, \cite{liang2020multiple} analyze the kernel interpolator under the general scaling regime $d = O(n^{\alpha}),~\alpha \in (0,1)$ and demonstrate the multiple descent behavior of the learning risk. By making Assumption \ref{assm:dimen}, we adopt a similar dimension scaling regime and analyze the risk behavior of the MNLS estimator under the two-layer neural network setting.

\begin{assm}\label{assm:x_dis}
Define $\Sigma = \mathbb{E}(xx^T)$ to be the covaraince matrix. We assume $x= \Sigma^{1/2}u$, where $u \in \mathbb{R}^d$ is a random vector and each entry is i.i.d subgaussian with $0$ mean and unit variance. In addition, denote $\{\lambda_i\}_{i=1}^d$ as the eigenvalues of $\Sigma$, and assume $\lambda_i \propto i^{-\gamma},~\gamma > 1$.
\end{assm}
Assumption \ref{assm:x_dis} is a common assumption used in analyzing the statistical learning risk~\citep[see, e.g.,][]{bartlett2020benign,jacot2020implicit,li2020benign}. In particular, when $d \rightarrow \infty$, Assumption \ref{assm:x_dis} indicates that $\Sigma$ is of trace-class as $\textnormal{Tr}(\Sigma) < \infty$. Assumptions \ref{assm:dimen} and \ref{assm:x_dis} together assume that the covariate $x$ lives in a high dimensional setting with $d = O(n^{\alpha})$. However, we require the covariate $x$ to have a low effective dimension, since $\mathbb{E}_{x}\|x\|_2^2 = \textnormal{Tr}(\Sigma) < \infty$. The idea of Assumptions \ref{assm:dimen} and \ref{assm:x_dis} is that the data used in practice (such as image and text) often exhibits a low effective dimensional representation property.

\begin{assm}\label{assm:feature_noise}
Given a training sample $\mathcal{D} = \{(x_i,y_i)\}_{i=1}^n$, we assume that the each of the sample covariates $x_i$ is corrupted by some i.i.d noise $\xi_i \sim \xi$, where $\xi \in \mathbb{R}^d$ and each entry of $\xi$ is i.i.d subgaussian with $0$ mean and variance $\sigma_{\xi}^2$. Furthermore, we assume $\sigma_{\xi}^2 = O(d^{-\zeta}), \zeta \geq 1$.
\end{assm}
Assumption \ref{assm:feature_noise} is motivated by the well-known fact that adding noise can help improve generalization performance. For example, in \cite{bishop1995training}, it is pointed out that training with noisy data amounts to adding regularization, and hence significantly improves prediction. In \cite{bartlett2020benign}, the influence of the covariate noise is studied under an overparameterized linear regression regime. They show that benign overfitting occurs when feature noise and the spectrum of the covariance operator both display exponential decay. Moreover, covariate noise has been further explored under the linear factor regression \citep{bunea2020interpolation} and random feature regression settings \citep{li2020benign}. Both demonstrate that the noise in the covariate can serve as implicit regularization and the interpolator can achieve near-optimal prediction accuracy under suitable conditions on the covariate noise.

In light of the effect of covariate noise, we study the benign overfitting phenomenon under the neural network setting by assuming that our sample covariate $x_i$ contains extra noise. Assumptions \ref{assm:feature_noise} details how the covariate is affected by the noise $\xi$. Based on Assmptions \ref{assm:x_dis} and \ref{assm:feature_noise}, we can see that the corrupted covariates $x_i + \xi_i$ can be written as
\begin{IEEEeqnarray}{rCl}
x_i + \xi_i = \Sigma_{\xi}^{1/2}u_i,~~~~\Sigma_{\xi} = \Sigma + \sigma_{\xi}^2I_{d} . \label{eqn:sigmaxi_def}
\end{IEEEeqnarray}
The requirement of $\zeta \geq 1$ is to ensure that $\text{Tr}(\Sigma_{\xi}) < \infty$ as $d \rightarrow \infty$ \footnote{We would like to point out that our analysis also applies to the case where $\zeta < 1$. In this case, however, $\text{Tr}(\Sigma_{\xi}) \rightarrow \infty$ as $d \rightarrow \infty,$ and $\Sigma_{\xi}$ is dominated by the covariate noise $\xi$. As a result, we would be regressing on pure noise, which is not interesting.}.

Since our training sample is now corrupted with $\xi$, the feature vectors will also contain the noise 
\begin{IEEEeqnarray}{rCl}
\mathbf{z}_{x_i+\xi_i} = \sigma\left(\mathbf{W}(x_i +\xi_i)\right) = D_{x_i+\xi_i}\mathbf{W}(x_i+\xi_i),~~ \mathbf{Z}_{\xi} = [\mathbf{z}_{x_1+\xi_1},\dots,\mathbf{z}_{x_n+\xi_n}]^T. \label{eqn:Zxi_def}
\end{IEEEeqnarray}

\subsection{Variance}
In this section, we analyze the behavior of the variance $\mathbf{V}_R$. Theorem \ref{theo:var_upper} (proof in Appendix \ref{sec:var_proof}) provides an insight on when the MNLS estimator $\tilde{\beta}$ can generalize well for the overparameterized model in the high dimensional setting ($s >n > d$)\footnote{Note that the condition $n >d$ is due to Assumption \ref{assm:dimen}.}.

\begin{restatable}{theorem}{var}\label{theo:var_upper}
Given Assumptions \ref{assm:regre}--\ref{assm:feature_noise} and assuming the high dimensional overfitting regime ($s > n > d$), let $\tilde{\beta}$ be the MNLS estimator defined in Eq.~(\ref{eqn:mnls}), $\{\lambda_i\}_{i=1}^d$ be the eigenvalues of the covariance matrix $\Sigma$ in descending order, and $\lambda_i^{\xi} = \lambda_i+\sigma_{\xi}^2$. Suppose there exists a $k^* \in [d]$, such that
\begin{IEEEeqnarray}{rCl}
\sum_{i>k^*}^d\frac{\lambda_i^{\xi}}{\lambda_{k^*}^{\xi}} \geq \frac{1}{b}d. \label{assm:noisy}
\end{IEEEeqnarray}  
Then with probability $1-4e^{-d/c}-4e^{-s/c}$, we have 
\begin{IEEEeqnarray}{rCl}
\mathbf{V}_R \leq c\sigma_0^2\textnormal{Tr}\left(\Sigma\right) \frac{s}{nd}, \label{eqn:var_upper}
\end{IEEEeqnarray}   
where $b,c >1$ are universal constants.
\end{restatable}
Theorem \ref{theo:var_upper} provides a finite sample convergence bound for the variance of the MNLS estimator $\tilde{\beta}$. It demonstrates that as long as the order of $s$ is not larger than $O(nd)$, the variance can decay to zero.

\paragraph{Remark 1} Before we further analyze the upper bound, we discuss our key condition: Eq.~(\ref{assm:noisy}). To simplify our discussion, we temporarily assume that $\zeta = 1$, so $\sigma_{\xi}^2 = 1/d$. Recall that we have assumed that $\lambda_i = O(i^{-\gamma})$ and $\gamma > 1 = \zeta$. Condition~(\ref{assm:noisy}) is equivalent to \[\sum_{i> k^*}^d \frac{\lambda_i^{\xi} }{\lambda_{k^*}^{\xi}}= \sum_{i> k^*}^d \frac{\lambda_i + \sigma_{\xi}^2}{\lambda_{k^*}+\sigma_{\xi}^2} \geq \frac{1}{b}d.\]
Since $\gamma > \zeta$, there must exist a $k^* \in [d]$ such that $\lambda_{k^*} \leq \sigma_{\xi}^2$. In this case, for all $k^*< i \leq d$, \[\dfrac{\lambda_i + \sigma_{\xi}^2}{  \lambda_{k^*}+\sigma_{\xi}^2}  \geq \frac{1}{2}.\]

As a result, $\sum_{i> k^*}^d(\lambda_i^{\xi}/\lambda_{k^*}^{\xi}) \geq \frac{1}{2}(d-k^*)$. In addition, if $k^* \ll d$, we will have $\frac{1}{2}(d-k^*) \geq \frac{1}{b}d$ for some constant $b$. Therefore, condition~(\ref{assm:noisy}) is equivalent to requiring that $k^*\in [d]$ such that $k^* \ll d$ and $\lambda_{k^*} \leq \sigma_{\xi}^2 = 1/d$.

Based on the value of $\gamma$, we split the discussion into three scenarios:
\begin{itemize}
    \item[$1$.] $\gamma = \infty$: $\Sigma$ has finite rank. In this case there is some $r$ such that $\lambda_i = 0$ for $i > r$. As such, let $k^* = r$, and we have $\sum_{i> k^*}^d(\lambda_i^{\xi}/\lambda_{k^*}^{\xi}) = (d-r) \geq \frac{1}{b}d$ for some constant $b$;
    
    \item[$2$.] $\gamma \propto i$: The spectrum of $\Sigma$ decays exponentially, i.e.,~$\lambda_i = O(\exp(-i))$. If we let $k^* = \log d$, it is easy to see that $\lambda_{k^*} = \frac{1}{d} \leq \sigma_{\xi}^2$. We therefore have $\sum_{i> k^*}^d(\lambda_i^{\xi}/\lambda_{k^*}^{\xi}) \geq \frac{1}{2}(d-\log d) \geq \frac{1}{b}d$;
    
    \item[$3$.] $\gamma$ is a constant: $\Sigma$ has polynomial decay, i.e.,~$\lambda_i = O(i^{-\gamma})$. If we let $k^* = d^{1/\gamma}$, we have $\lambda_{k^*} = (d^{1/\gamma} )^{-\gamma} =  1/d \leq \sigma_{\xi}^2$. Therefore, $\sum_{i> k^*}^d(\lambda_i^{\xi}/\lambda_{k^*}^{\xi}) \geq \frac{1}{2}(d- d^{1/\gamma}) \geq \frac{1}{b}d$.
\end{itemize}

The analysis for the case where $\zeta > 1$ is similar. The key point is that there exists $k^* \ll d$ such that $\lambda_{k^*} \leq \sigma_{\xi}^2$. This requirement is guaranteed to hold if the noise $\xi$ has the property that $\zeta <\gamma$, i.e., as long as the decay rate of $\sigma_{\xi}$ is lower than that of $\lambda_i$, there is a $k^*$ such that Eq.~(\ref{assm:noisy}) holds. Therefore, we can see that Eq.~(\ref{assm:noisy}) is a mild requirement on the covariate noise $\xi$.

We now discuss the upper bound for the variance $\mathbf{V}_R$. It is easy to see that the convergence rate of $\mathbf{V}_R$ is governed by $O(s/nd) = O(s/n^{1+\alpha})$. As such, if we choose $s = O(n^{\kappa})$ with $\kappa < 1+\alpha$, the variance decays to zero asymptotically when $n \rightarrow \infty$. In this case, even if $s \gg n$ in the sense that $\lim_{n\rightarrow \infty} s/n \rightarrow \infty$ (i.e.,~the heavily overparameterized setting), as long as $\kappa < 1+ \alpha$, $\mathbf{V}_R$ still converges to zero. 

Moreover, if the data dimension $d$ is constant with respect to $n$, we can see that in the overparameterized case $s > n$, the variance as well as the excess learning risk of the MNLS estimator do not converge. We remark that \cite{rakhlin2019consistency} show that the MNLS estimator for Laplace kernel is \textit{not} consistent (i.e.,~the excess risk does not converges to zero with $n \rightarrow \infty$) if $d$ is constant with respect to $n$. It is interesting to see that our results yield similar findings to \cite{rakhlin2019consistency} for the two-layer ReLU neural network. Finally, we point out that Theorem \ref{theo:var_upper} also provides us the convergence rate of the variance term, which we will discuss in Section \ref{sec:ex_risk}.

\subsubsection{Lower Bound}
Having established the finite sample upper bound of the variance, our next theorem (proof in Appendix \ref{sec:lower_proof}) provides a lower bound, which demonstrates that our upper bound is tight. 

\begin{restatable}{theorem}{low}\label{theo:var_lower}
Given Assumptions \ref{assm:regre}--\ref{assm:feature_noise} and assuming $s > n > d$, let $\tilde{\beta}$ be the MNLS estimator defined in Eq.~(\ref{eqn:mnls}), for some universal constant $c>1$. With probability greater than $1-5e^{-d/c}-2e^{-s/c}$, we have 
\begin{IEEEeqnarray}{rCl}
\mathbf{V}_R \geq \frac{1}{c}\sigma_0^2\textnormal{Tr}\left(\Sigma\right)\frac{s}{nd}. \label{eqn:var_lower}
\end{IEEEeqnarray}   
\end{restatable}
\paragraph{Remark 2} Although the lower bound looks similar to the upper bound, the universal constant is different. Furthermore, the lower bound does not require the existence of $k^*$ in Eq.~(\ref{assm:noisy}).

Theorem \ref{theo:var_lower} indicates that our upper bound on the variance is tight. Together with Theorem \ref{theo:var_upper}, it shows that the variance decays as $O(s/nd)$. The lower bound also demonstrates that we cannot arbitrarily increase the number of parameters: the order of $s$ has to remain $O(nd)$ to ensure convergence. Since $d < n$, our theory predicts that the number of the parameter cannot exceed the order of $n^2$, i.e.,~if we choose $s = \Omega(n^2)$, the variance and hence the learning risk will start to grow again. 

We remark that a line of recent empirical work has shown that the excess learning risk curve exhibits  multiple descent, instead of double descent,  in many learning settings, including linear regression \citep{nakkiran2020optimal}, random Fourier feature regression \citep{d2020triple}, kernel regression \citep{liang2020multiple} and two-layer neural networks \citep{adlam2020neural}. In recent theoretical work, \cite{chen2020multiple} and \cite{li2020benign}  provide justifications for  multiple descent behavior in the linear regression and random feature settings, respectively. Our upper and lower bounds on the variance indicate that it is possible for a two-layer ReLU neural network to obtain a multiple descent learning risk curve, matching the recent empirical findings.


\subsection{Bias} \label{sec:bias}
In this section, we present the upper bound for the bias of the MNLS estimator. Our next theorem (proof in Appendix \ref{sec:bias_proof}) shows that under mild conditions, the bias also converges to zero for heavily overparameterized models in the high dimensional setting ($s > n >d$) .

\begin{restatable}{theorem}{bias}\label{theo:bias_upper}
Given Assumptions \ref{assm:regre}--\ref{assm:feature_noise} and assuming $s > n > d$, for any $\delta \in (0,1)$ and a universal constant $c> 1 $, with probability greater than $1- \delta -4e^{-d/c}-2e^{-s/c}$, we have 
\begin{IEEEeqnarray}{rCl}
\mathbf{B}_R \leq c \left\{\sqrt{\frac{1}{n}\log\frac{s}{\delta}} +\frac{d^2\sigma_{\xi}}{n} \right\}. \label{eqn:bias_upper}
\end{IEEEeqnarray}   
\end{restatable}

Theorem \ref{theo:bias_upper} demonstrates that the bias term $\mathbf{B}_R$ can converge, where the convergence rate depends on the covariate dimension $d$ and the noise level $\sigma_{\xi}$. Given $d = O(n^{\alpha})$ and $\sigma_{\xi}^2 = O(d^{-\zeta})$, it is easy to see that the bias $\mathbf{B}_R$ is governed by the rate $O(\sqrt{\log n/n} + n^{2\alpha-1-\alpha\zeta/2})$. Based on the values of $\alpha$ and $\zeta$, we have the following three scenarios:

\begin{itemize}
    \item[B.$1$] $\alpha \in (0, 1/2]$: In this case, $\mathbf{B}_R$ is governed by $O(\sqrt{\log n/n}+ \sigma_{\xi}) = O(\sqrt{\log n/n} + d^{-\zeta/2})$. Since $\zeta \geq 1$, $\mathbf{B}_R$ is guaranteed to converge;
    
    \item[B.$2$] $\zeta \geq 2$: When $\zeta \geq 2$, $d^2\sigma_{\xi} \leq d$. Hence, $\mathbf{B}_R$ is on the order of $O(\sqrt{\log n/n}+d/n)$. Then $\alpha< 1$ implies that $\mathbf{B}_R$ converges;
    
    \item[B.$3$] $\alpha \in (1/2,1) ~\&~ \zeta \in [1,2)$: In this case, $\mathbf{B}_R$ converges at rate $O(\sqrt{\log n/n}+ n^{2\alpha-1-\alpha\zeta/2})$. Therefore, convergence of the bias amounts to requiring that $2\alpha-1-\alpha\zeta/2 < 0$. This is equivalent to $\zeta > 4 - 2/\alpha$.
\end{itemize}

In summary, Theorem \ref{theo:bias_upper} reveals that the bias term converges under mild requirements. In particular, the convergence depends on the interplay between the following three quantities: $\alpha$ (the rate at which dimension grows with $n$), $\gamma$ (the decay rate of the spectrum of $\Sigma$) and $\zeta$ (the decay rate of the noise $\xi$). For instance, if the spectrum of $\Sigma$ exhibits fast decay ($\gamma \in (2, \infty)$), and $2 \leq \zeta < \gamma$, the bias is guaranteed to converge according to B.$2$. Note that the requirement for $\zeta < \gamma$ is to ensure the convergence of $\mathbf{V}_R$ (see Remark 1).

On the other hand, even if the spectrum of $\Sigma$ has a slow decay rate where $\gamma$ is close to $1$ and $\zeta < \gamma$, as long as the dimension of the covariate is low ($\alpha \in (0,1/2]$), by B.$1$, we can still observe the convergence of $\mathbf{B}_R$. In the worst scenario where the spectrum exhibits slow decay ($\gamma$ is close to $1$) and the covariate has relatively high dimension $\alpha > 1/2$, $\mathbf{B}_R$ converges only in the regime where $\zeta > 4 - 2/\alpha$. For example, if $\alpha = 4/5$, we require $\gamma > \zeta > 5/4$ to observe the convergence of the bias. 

Finally, we remark that if the decay rate of the noise $\zeta$ is not benign, $\mathbf{B}_R$ can diverge. For instance, if both $\alpha$ and $\zeta$ are close to $1$, $\mathbf{B}_R$ is on the order of $O(n^{1/2})$, which diverges as we increase $n$.

\subsection{Convergence of The Excess Learning Risk and Near Minimax Optimality}\label{sec:ex_risk}
In this section, we discuss the convergence of the excess learning risk and show that the two-layer ReLU network MNLS estimator can achieve near optimal learning rate in the minimax sense.

Theorems \ref{theo:var_upper} and \ref{theo:bias_upper} together imply that the excess learning risk $R(\tilde{\beta})$ for the MNLS estimator $\tilde{\beta}$ can be upper bounded by \[R(\tilde{\beta})\leq O\left(\sqrt{\frac{1}{n}\log s}+ \frac{d^2\sigma_{\xi}}{n} + \frac{s}{nd}\right).\] In particular, the upper bound indicates that if $\zeta < \gamma$ and $\zeta > 4 -2/\alpha$, the overparameterized ReLU network can generalize well, i.e.,~we observe benign overfitting. From this we can see that the decay of the excess learning risk for the MNLS estimator depends on the interplay between the properties of the covariate $x$ (as represented by $\alpha$ and $\gamma$) and the size of the noise $\xi$ (as represented by $\zeta$). Depending on the values of $\zeta$, the excess learing risk can either converge or diverge. In the case that the covariate noise is added manually by the user, our results state that adding noise to the covariate can serve as an implicit regularization during training, and thereby leads to significant improvements in generalization performance. Our findings match with previous results where interpolating noisy data can achieve near-optimal generalization performance~\citep[see, e.g.,][]{bishop1995training,bartlett2020benign,li2020benign,richards2021asymptotics}.

We further remark that the upper bound not only reveals the conditions for benign overfitting, but also demonstrates the learning rate of the MNLS estimator. For example, if $\alpha = 0.5$, $\gamma > \zeta = 2$ and we choose $s = O(n^{1.25})$, the excess learning risk converges at rate $O(n^{-0.25})$. 

Finally, our results indicate that the interpolator can obtain near minimax optimal learning rate depending on the properties of the data and covariate noise (i.e.,~$\alpha$, $\gamma$ and $\zeta$). For example, if $\alpha = 0.5$ and $\gamma > \zeta = 2$, and we choose the overparameterized regime by letting $s = \Omega(n)$ and $s/n = c$ for a constant $c > 1$, the excess learning risk is now at the $O\left(\sqrt{\log n/ n}\right)$ rate. Recalling that the optimal learning rate in the minimax sense is $O(n^{-1/2})$ \citep{caponnetto2007optimal} for a typical non-linear regression, we conclude that the MNLS estimator can achieve near minimax optimal learning rate in the presence of the covariate noise $\xi$. 

\section{Discussion}\label{sec:dis}
By assuming the high dimensional setting and the existence of the covariate noise, our results characterize the conditions under which benign overfitting occurs for a two-layer ReLU neural network. We derive a finite sample excess risk bound and show that the excess risk of the MNLS estimator can vanish in $n$ for a wide range of high dimensional settings $d = n^{\alpha}, \alpha \in (0,1)$. Our analysis reveals that the interplay between the dimension parameter $\alpha$, the decay rate of the spectrum of the covariance $\gamma$, and the covariate noise decay rate $\zeta$ plays an important role in determining the generalization performance of overparameterized models. In particular, when data lives in a relatively low dimension $\alpha \leq 1/2$, or the spectrum and the covariate noise have fast decay $\gamma > \zeta \geq 2$, a heavily overparameterized neural network can still achieve optimal prediction accuracy. Beyond those regimes, obtaining optimal prediction accuracy for the overparameterized models requires the $\alpha, \gamma, \zeta$ to exhibit benign conditions.

In addition, we also demonstrate that the MNLS estimator can achieve various learning rates. In particular, under suitable conditions, an overparameterized ReLU network that interpolates training data can enjoy a near minimax optimal learning rate $O\left(\sqrt{\log n / n}\right)$. Finally, we observe that the excess learning risk for overparameterized ReLU networks starts to increase once the number of parameter $s$ is beyond the $O(n^2)$ order, which generalizes the double descent phenomenon in linear regression and other models.

We would also like to point out some limitations of our work. We currently assume that the data generating distribution and the covariate noise are subgaussian. The restriction is largely due to the concentration inequalities used only applying for subgaussian distributions. Therefore, we do not have a clear understanding of benign overfitting when both data generating distribution and covariate noise have heavy tailed distributions. Furthermore, our analysis applies to the regime where $d < n$, and not  when $d > n$, which may occur in some real-world applications.  

Having discussed our limitations, we believe that there are several interesting directions to investigate. First, it would be interesting to extend our analysis to the settings where the data and noise have heavy tailed distributions. Second, providing analysis in the $d > n$ regime is an important direction. In addition, it would be interesting to investigate how our results can extend to different activation functions such as the sigmoid function, softplus, etc., since they are widely used in practice. Finally, when we analyze the properties of the MNLS estimator, the first layer weights of the ReLU neural network are kept constant during training. We would like to understand the settings where all the parameters from the ReLU network are optimized, where benign overfitting is also observed.

\bibliography{ref}
\bibliographystyle{abbrvnat}
\newpage
\appendix
\section{Notation}
In the appendix, we will use $a_1,a_2,\dots,b_1,b_2,\dots,c_1,c_2,\dots > 1$ to represent universal constants. We also use $\|\mathbf{a}\|_2$ to denote the $l_2$ norm for a vector $\mathbf{a}$ and $\|\mathbf{A}\|_2$ denote the operator norm for matrix $\mathbf{A}$. In addition, for matrix $\mathbf{A}\in\mathbb{R}^{n\times n}$, we denote its eigenvalues as $\mu_{1}(\mathbf{A})\geq \dots,\geq \mu_{n}(\mathbf{A})$ in descending order. In particular, if the rank of $A$ is less than $n$, we use $\mu_{\min}(A) > 0$ to denote the least positive eigenvalue of $A$.

For a matrix $\mathbf{W}\in\mathbb{R}^{s\times d}$, we use $\mathbf{W}_i \in \mathbb{R}^s$ to denote its $i$-th column and $\mathbf{w}_i \in \mathbb{R}^d$ to denote its $i$-th row. We use $\mathbf{1}_s \in \mathbb{R}^s$ to denote the vector $[1,\dots,1]^T$ and $I_s \in \mathbb{R}^{s\times s}$ to denote the $s$ dimensional identity matrix. Given integer $s$, we use $[s]$ to denote the set $\{1,\dots,s\}$.
\section{Bias-Variance Decomposition} \label{sec:b-v}
\BiaVar*
\begin{proof}
We decompose the excess risk as follows:
\begin{IEEEeqnarray}{rCl}
R(\tilde{\beta}) &= & \mathbb{E}_{\pmb{\epsilon}} \left\{\mathbb{E}_{x,y}\left[(\hat{f}(x)-y)^2-  (f_*(x)-y)^2\right] \right\} \nonumber\\
&=&\mathbb{E}_{x,\pmb{\epsilon}}\left(\tilde{f}(x)- f_*(x)\right)^2 = \mathbb{E}_{x,\pmb{\epsilon}} \left(\mathbf{z}_{x}^T\tilde{\beta} - f_*(x) \right)^2, \nonumber \\
& =& \mathbb{E}_{x,\pmb{\epsilon}} \left\{\mathbf{z}_{x}^T(\mathbf{Z}^T\mathbf{Z})^{\dagger}\mathbf{Z}^TY - f_{\mathcal{H}}(x) + \left[f_{\mathcal{H}}(x)-f_*(x)\right]\right\}^2, \nonumber \\
& =& \mathbb{E}_{x,\pmb{\epsilon}} \left\{\mathbf{z}_{x}^T(\mathbf{Z}^T\mathbf{Z})^{\dagger}\mathbf{Z}^T\left(f_*(X)- f_{\mathcal{H}}(X) + f_{\mathcal{H}}(X) + \pmb{\epsilon}\right) - f_{\mathcal{H}}(x) + \left[f_{\mathcal{H}}(x)-f_*(x)\right]\right\}^2, \nonumber \\
& \leq & 3 \mathbb{E}_{x,\pmb{\epsilon}} \left\{\mathbf{z}_{x}^T(\mathbf{Z}^T\mathbf{Z})^{\dagger}\mathbf{Z}^T \left(f_{\mathcal{H}}(X) + \pmb{\epsilon} \right)- f_{\mathcal{H}}(x)  \right\}^2 := A \nonumber \\
&& + 3\left\{\mathbb{E}_{x} \left\{\mathbf{z}_{x}^T(\mathbf{Z}^T\mathbf{Z})^{\dagger}\mathbf{Z}^T \left(f_*(X) - f_{\mathcal{H}}(X) \right) \right\}^2 + \mathbb{E}_{x} \left(f_*(x) - f_{\mathcal{H}}(x)\right)^2 \right\}:= \mathbf{M}_R.\nonumber
\end{IEEEeqnarray}
Now we can see that the risk has been decomposed into the $A$ term and the misspecification error term $\mathbf{M}_R$. For the $A$ term, we have
\begin{IEEEeqnarray}{rCl}
A &= & \mathbb{E}_{x,\pmb{\epsilon}} \bigg\{\mathbf{z}_x^T(\mathbf{Z}^T\mathbf{Z})^{\dagger}\mathbf{Z}^T( f_{\mathcal{H}}(X) + \pmb{\epsilon})- \mathbf{z}_x^T\beta_{\mathcal{H}}\bigg\}^2 \nonumber \\
& = & \mathbb{E}_{x,\pmb{\epsilon}} \bigg\{\mathbf{z}_x^T(\mathbf{Z}^T\mathbf{Z})^{\dagger}\mathbf{Z}^T \pmb{\epsilon} + \mathbf{z}_x^T\left((\mathbf{Z}^T\mathbf{Z})^{\dagger}\mathbf{Z}^T\mathbf{Z}-I\right)\beta_{\mathcal{H}}\bigg\}^2, \nonumber \\
& = & \int_{\mathcal{X}} \left\|\mathbf{z}_x^T\left[(\mathbf{Z}^T\mathbf{Z})^{\dagger}\mathbf{Z}^T\mathbf{Z} - I\right]\beta_{\mathcal{H}}\right\|^2d\rho(x) := \mathbf{B}_R \nonumber\\
&&+ \int_{\mathcal{X}} \mathbb{E}_{\pmb{\epsilon}}\left\| \mathbf{z}_x^T(\mathbf{Z}^T\mathbf{Z})^{\dagger}\mathbf{Z}^T\pmb{\epsilon}\right\|^2 d\rho(x):= \mathbf{V}_R.\nonumber
\end{IEEEeqnarray}
We then replace $\pmb{\epsilon} $ with $Y -f_*(X)$.
\end{proof}

\section{Proof of Theorem \ref{theo:var_upper}}\label{sec:var_proof}
\var*
\begin{proof}
To upper bound the variance, we first note that some simple algebra yields a basic result:
\begin{IEEEeqnarray}{rCl}
\mathbf{V}_{R} &= & \mathbb{E}_{x}\bigg\{\mathbb{E}_{\pmb{\epsilon}}\left[ \mathbf{z}_{x}^T(\mathbf{Z}_{\xi}^T\mathbf{Z}_{\xi})^{\dagger}\mathbf{Z}_{\xi}^T\pmb{\epsilon}\pmb{\epsilon}^T\mathbf{Z}_{\xi}(\mathbf{Z}_{\xi}^T\mathbf{Z}_{\xi})^{\dagger}\mathbf{z}_{x}\right]\bigg\},\nonumber \\
& = & \sigma_0^2 \mathbb{E}_{x}\left\{ \mathbf{z}_{x}^T(\mathbf{Z}_{\xi}^T\mathbf{Z}_{\xi})^{\dagger}\mathbf{Z}_{\xi}^T\mathbf{Z}_{\xi}(\mathbf{Z}_{\xi}^T\mathbf{Z}_{\xi})^{\dagger}\mathbf{z}_{x} \right\},\nonumber \\
& =&\sigma_0^2 \mathbb{E}_x\left(\textnormal{Tr}\bigg[\mathbf{z}_{x}^T(\mathbf{Z}_{\xi}^T\mathbf{Z}_{\xi})^{\dagger}\mathbf{z}_{x}\bigg]\right).\nonumber\\
&=& \frac{\sigma_0^2}{n}\mathbb{E}_x\left(\textnormal{Tr}\bigg[\mathbf{z}_{x}^T(\frac{1}{n}\mathbf{Z}_{\xi}^T\mathbf{Z}_{\xi})^{\dagger}\mathbf{z}_{x}\bigg]\right). \nonumber
\end{IEEEeqnarray}
Therefore, we need to study the least positive eigenvalue $\mu_{\min}\left(\frac{1}{n}\mathbf{Z}_{\xi}^T\mathbf{Z}_{\xi}\right)$. By Lemma \ref{lma:zz_ww_diff}, we know that with probability greater than  $1-2e^{-s/b_1}-2e^{-d/b_1}$, 
\[\left\| \frac{1}{n}\mathbf{Z}_{\xi}^T\mathbf{Z}_{\xi} - \mathbf{W}\Sigma_{\xi}\mathbf{W}\right\|_2 = O\left(\sqrt{\frac{d}{n}} d\sigma_w^2\right).\]
We hence can lower bound the least positive eigenvalue of $\frac{1}{n}\mathbf{Z}_{\xi}^T\mathbf{Z}_{\xi}$ as \[\mu_{\min}\left(\frac{1}{n}\mathbf{Z}_{\xi}^T\mathbf{Z}_{\xi}\right) \geq \mu_{\min}\left(\mathbf{W}\Sigma_{\xi}\mathbf{W} \right)- b_2\sqrt{\frac{d}{n}} d\sigma_w^2.\]
Let $A = \mathbf{W}\Sigma_{\xi}\mathbf{W}^T = \sum_{i=1}^d \lambda_i^{\xi} \mathbf{W}_i\mathbf{W}_i^T$, where $\lambda_i^{\xi} = \lambda_i + \sigma_{\xi}^2$. We also denote $A_k = \sum_{i>k}^d \lambda_i^{\xi} \mathbf{W}_i\mathbf{W}_i^T$. By Lemma \ref{lma:w_sigam_w}, we have with probability greater than $1-2e^{-d/\eta}$, \[\mu_{\min}(A) \geq \mu_{\min}(A_k)\geq (1-\frac{1}{\sqrt{\eta}})\sigma_w^2 \left( \sum_{i>k}\lambda_{i}^{\xi} - \frac{1}{\sqrt{\eta}}\lambda_k^{\xi}d\right).\] If we choose $k = k^*$, then we have $ \left( \sum_{i>k}\lambda_{i}^{\xi} - \frac{1}{\sqrt{\eta}}\lambda_k^{\xi}d\right) \geq \frac{1}{b_3} d$. We hence have \[\mu_{\min}\left(\frac{1}{n}\mathbf{Z}_{\xi}^T\mathbf{Z}_{\xi}\right) \geq \frac{1}{b_3}d\sigma_w^2 - b_2\sqrt{\frac{d}{n}} d\sigma_w^2 \geq \frac{1}{b_4}d\sigma_w^2,\] provided that $n$ is large enough.

We can now upper bound $\mathbf{V}_R$ as 
\begin{IEEEeqnarray}{rCl}
\mathbf{V}_{R} &\leq &b_4\frac{\sigma_0^2}{nd\sigma_w^2}\mathbb{E}_x\left(\textnormal{Tr}(\mathbf{z}_{x}^T\mathbf{z}_{x})\right) \nonumber\\
&\leq &b_4 \frac{\sigma_0^2}{nd\sigma_w^2} s\sigma_w^2 \mathbb{E}_x(\|x\|_2^2) \nonumber \\
&\leq & b_4\frac{s}{nd}\sigma_0^2\textnormal{Tr}(\Sigma)\nonumber\\
&\leq & b_5\frac{s}{nd}\sigma_0^2. \nonumber
\end{IEEEeqnarray}
Letting $c = \max\{b_1,\dots\}$ yields the final result.
\end{proof}

\begin{lma}\label{lma:zz_ww_diff}
Let $\Sigma_{\xi}$ and $\mathbf{Z}_{\xi}$ be defined as in Eq.~(\ref{eqn:sigmaxi_def}) and Eq.~(\ref{eqn:Zxi_def}) respectively, then we have with probability greater than $1-2e^{-s/c}-2e^{-d/c}$, \[\left\| \frac{1}{n}\mathbf{Z}_{\xi}^T\mathbf{Z}_{\xi} - \mathbf{W}\Sigma_{\xi}\mathbf{W}\right\|_2 = O\left(\sqrt{\frac{d}{n}} d\sigma_w^2\right).\]
\end{lma}

\begin{proof}
We first notice that 
\begin{IEEEeqnarray}{rCl}
\frac{1}{n}\mathbf{Z}_{\xi}^T\mathbf{Z}_{\xi} - \mathbf{W}\Sigma_{\xi}\mathbf{W} &=& \frac{1}{n}\sum_{i=1}^n\mathbf{z}_{x_i+\xi_i}\mathbf{z}_{x_i+\xi_i}^T - \mathbf{W}\Sigma_{\xi}\mathbf{W}\nonumber\\
&=& \frac{1}{n}\sum_{i=1}^n\mathbf{z}_{x_i+\xi_i}\mathbf{z}_{x_i+\xi_i}^T - \frac{1}{n}\sum_{i=1}^n \mathbf{W}(x_i+\xi_i)(x_i+\xi_i)^T\mathbf{W}^T \label{eqn:z_w}\\
&& + \frac{1}{n}\sum_{i=1}^n \mathbf{W}(x_i+\xi_i)(x_i+\xi_i)^T\mathbf{W}^T- \mathbf{W}\Sigma_{\xi}\mathbf{W}. \label{eqn:w_sigma}
\end{IEEEeqnarray}

For Eq.~(\ref{eqn:z_w}), we can further write as
\begin{IEEEeqnarray}{rCl}
&&\frac{1}{n}\sum_{i=1}^n\mathbf{z}_{x_i+\xi_i}\mathbf{z}_{x_i+\xi_i}^T - \frac{1}{n}\sum_{i=1}^n \mathbf{W}(x_i+\xi_i)(x_i+\xi_i)^T\mathbf{W}^T \nonumber\\
&=&\frac{1}{n}\sum_{i=1}^n\mathbf{z}_{x_i+\xi_i}\mathbf{z}_{x_i+\xi_i}^T - \frac{1}{n}\sum_{i=1}^n\mathbf{z}_{x_i+\xi_i}(\mathbf{W}(x_i+\xi_i))^T \nonumber\\ 
&&+ \frac{1}{n}\sum_{i=1}^n\mathbf{z}_{x_i+\xi_i}(\mathbf{W}(x_i+\xi_i))^T -\frac{1}{n}\sum_{i=1}^n \mathbf{W}(x_i+\xi_i)(x_i+\xi_i)^T\mathbf{W}^T \nonumber\\
&=& \frac{1}{n}\sum_{i=1}^n\mathbf{z}_{x_i+\xi_i}\left(\mathbf{z}_{x_i+\xi_i}-\mathbf{W}(x_i+\xi_i)\right)^T + \frac{1}{n}\sum_{i=1}^n\left(\mathbf{z}_{x_i+\xi_i}-\mathbf{W}(x_i+\xi_i)\right) \left(\mathbf{W}(x_i+\xi_i)\right)^T\nonumber\\
&=& \frac{1}{n}\sum_{i=1}^n\mathbf{z}_{x_i+\xi_i}\mathbf{z}_{x_i+\xi_i}^T +\frac{1}{n}\sum_{i=1}^n\mathbf{z}_{x_i+\xi_i}(\mathbf{W}(x_i+\xi_i))^T \nonumber.
\end{IEEEeqnarray}
Note for the last step, we used the fact that $\sigma(x)-x = \sigma(-x)$ and $\mathbf{z}_{x_i+\epsilon_i} = \sigma(\mathbf{W}(x_i+\epsilon_i))$. In addition, for  a Gaussian random matrix $\mathbf{W}$ with mean $0$, $-\mathbf{W}$ and $\mathbf{W}$ have exactly the same distribution.




Hence, we can upper bound the norm of Eq.~(\ref{eqn:z_w}) with
\begin{IEEEeqnarray*}{rCl}
&&\left\|\frac{1}{n}\sum_{i=1}^n\mathbf{z}_{x_i+\xi_i}\mathbf{z}_{x_i+\xi_i}^T +\frac{1}{n}\sum_{i=1}^n\mathbf{z}_{x_i+\xi_i}(\mathbf{W}(x_i+\xi_i))^T\right\|_2 \\
&\leq& \left\| \frac{1}{n}\sum_{i=1}^n\mathbf{z}_{x_i+\xi_i}\mathbf{z}_{x_i+\xi_i}^T\right\|_2 + \left\| \frac{1}{n}\sum_{i=1}^n\mathbf{z}_{x_i+\xi_i}(\mathbf{W}(x_i+\xi_i))^T\right\|_2\\
&=& \left\| \frac{1}{n}\sum_{i=1}^nD_{x_i+\xi_i}\mathbf{W}(x_i+\xi_i)(x_i+\xi_i)^T\mathbf{W}^TD_{x_i+\xi_i}\right\|_2 \\
&&+ \left\| \frac{1}{n}\sum_{i=1}^n\mathbf{W}(x_i+\xi_i)(x_i+\xi_i)^T\mathbf{W}^TD_{x_i+\xi_i}\right\|_2\\
&\leq & 2 \left\| \frac{1}{n}\sum_{i=1}^n\mathbf{W}(x_i+\xi_i)(x_i+\xi_i)^T\mathbf{W}^T\right\|_2\\
&=&\frac{2}{n}\left\|\mathbf{W}\Sigma_{\xi}^{1/2}U^TU\Sigma_{\xi}^{1/2} \mathbf{W}^T \right\|_2\\
&\leq & 2\frac{d}{n}\left\|\mathbf{W}\Sigma_{\xi}\mathbf{W}^T \right\|_2\\
&\leq & 2\frac{d}{n}d\sigma_w^2.
\end{IEEEeqnarray*}
Note that for the second last step, we apply Lemma \ref{lma:eigen_A_bd} to $U^TU$ with $U = [u_1,\dots, u_n]^T$, where each $u_i$ is a subgaussian random vector defined in Eq.~(\ref{eqn:sigmaxi_def}). For the last step we apply Lemma \ref{lma:eigen_A_bd} to $\mathbf{W}\Sigma_{\xi}\mathbf{W}^T = \sum_{i=1}^d (\lambda_i+\sigma_{\xi}^2)\mathbf{W}_i\mathbf{W}_i^T$.

We now study Eq.~(\ref{eqn:w_sigma}),
\begin{IEEEeqnarray*}{rCl}
&&\frac{1}{n}\sum_{i=1}^n \mathbf{W}(x_i+\xi_i)(x_i+\xi_i)^T\mathbf{W}^T- \mathbf{W}\Sigma_{\xi}\mathbf{W} \\
&=& \frac{1}{n}\sum_{i=1}^n \mathbf{W}\Sigma_{\xi}^{1/2}u_iu_i^T\Sigma_{\xi}^{1/2}\mathbf{W}^T - \mathbf{W}\Sigma_{\xi}\mathbf{W}\\
&=& \mathbf{W}\Sigma_{\xi}^{1/2} \left(\frac{1}{n}U^TU -I_{d} \right)\Sigma_{\xi}^{1/2}\mathbf{W}^T. 
\end{IEEEeqnarray*}
Let $\mathbf{v}\in \mathbb{R}^{d}$ be a unit vector. We then have \[\mathbf{v}^T\left(\frac{1}{n}U^TU -I_{d} \right)\mathbf{v} = \frac{1}{n} \sum_{i=1}^n \left\{\left(\sum_{i=1}^d \mathbf{v}_j u_{i,j}\right)^2-1\right\}.\]

For each $i$, $\left(\sum_{i=1}^d \mathbf{v}_j u_{i,j}\right)^2-1$ is a subexponential random variable with $0$ mean. We apply Lemma \ref{lma: sub_exp_sum} to obtain with probability greater than $1-2e^{-t}$, \[\frac{1}{n} \sum_{i=1}^n \left\{\left(\sum_{i=1}^d \mathbf{v}_j u_{i,j}\right)^2-1\right\} \leq b_1 \left( \frac{t}{n}+ \sqrt{\frac{t}{n}}\right).\]
Similar to the proof of Lemma \ref{lma:eigen_A_bd}, we use  the $\epsilon$-net argument from Lemma \ref{lma:e-net}, to yield that with probability greater than $1-2e^{-d/b_{2}}$, \[\left\|\frac{1}{n}U^TU -I_{d}  \right\|_2 \leq b_3 \left(\frac{d}{n} + \sqrt{\frac{d}{n}}\right) \leq b_3\sqrt{\frac{d}{n}} .\]

Therefore, we have 
\begin{IEEEeqnarray*}{rCl}
\left\|\frac{1}{n}\sum_{i=1}^n \mathbf{W}(x_i+\xi_i)(x_i+\xi_i)^T\mathbf{W}^T- \mathbf{W}\Sigma_{\xi}\mathbf{W} \right\|_2
\leq  \left\|\mathbf{W}\Sigma_{\xi}\mathbf{W} \right\|_2
\leq  b_{4}\sqrt{\frac{d}{n}}d\sigma_w^2.
\end{IEEEeqnarray*}
Note that we used Lemma \ref{lma:eigen_A_bd} to upper bound the norm of $\mathbf{W}\Sigma_{\xi}\mathbf{W} = \sum_{i=1}^d (\lambda_i+\sigma_{\xi}^2) \mathbf{W}_i\mathbf{W}_i^T$.

Combining Eq.~(\ref{eqn:z_w}) and Eq.~(\ref{eqn:w_sigma}) together, and letting $c = \max\{b_1,\dots\}$ we have with probability greater than $1-2e^{-s/c}-2e^{-d/c}$,  \[\left\| \frac{1}{n}\mathbf{Z}_{\xi}^T\mathbf{Z}_{\xi} - \mathbf{W}\Sigma_{\xi}\mathbf{W}\right\|_2 = O\left(\sqrt{\frac{d}{n}} d\sigma_w^2\right).\]
\end{proof}

The following lemma provides the upper and lower bounds on the eigenspectrum of $\mathbf{W}\Sigma_{\xi}\mathbf{W}$. The lemma can be seen as a refined version of Lemma \ref{lma:eigen_A_bd} in Section \ref{sec:eigen_A}.
\begin{lma}\label{lma:w_sigam_w}
Let $A = \mathbf{W}\Sigma_{\xi}\mathbf{W}^T$, for some universal constant $c_1, c_2 > 1$, with probability greater than $1-2e^{-d/\eta}$, we have \begin{IEEEeqnarray}{rCl}
\sigma_w^2\left((1- \frac{1}{c_1\sqrt{\eta}})\sum_{i=1}^d \lambda_i^{\xi} - \frac{c_2}{\sqrt{\eta}}\lambda_1^{\xi}d\right) \leq \mu_{\min}(A) \leq \mu_1(A) \leq \sigma_w^2\left((1+ \frac{1}{c_1\sqrt{\eta}})\sum_{i=1}^d \lambda_i^{\xi} + \frac{c_2}{\sqrt{\eta}}\lambda_1^{\xi}d\right).  \nonumber
\end{IEEEeqnarray}
Let $A_k = \sum_{i>k}^d \lambda_i^{\xi}\mathbf{W}_i\mathbf{W}_i^T$, for some $1\leq k \leq d$ with the same probability bound we have 
\begin{IEEEeqnarray}{rCl}
\sigma_w^2\left((1- \frac{1}{c_1\sqrt{\eta}})\sum_{i>k}^d \lambda_i^{\xi} - \frac{c_2}{\sqrt{\eta}}\lambda_1^{\xi}d\right) \leq \mu_{\min}(A) \leq \mu_1(A) \leq \sigma_w^2\left((1+ \frac{1}{c_1\sqrt{\eta}})\sum_{i>k}^d \lambda_i^{\xi} + \frac{c_2}{\sqrt{\eta}}\lambda_1^{\xi}d\right).  \nonumber
\end{IEEEeqnarray}
\end{lma}

\begin{proof}
We write $A = \sum_{i=1}^d \lambda_i^{\xi}\mathbf{W}_i\mathbf{W}_i^T$, where we recall $\mathbf{W}_i \in \mathbb{R}^s$ is a subgaussian random vector and $\lambda_i^{\xi} =\lambda_{i}+\sigma_{\xi}^2$. For any unit vector $\mathbf{v} \in \mathbb{R}^s$, we have $\mathbf{v}^T\mathbf{W}_i$ is $\sigma_w^2$-subgaussian random variable. Hence for any $\mathbf{v}$, $\mathbf{v}^TA\mathbf{v} = \sum_{i=1}^n \lambda_i^{\xi} (\mathbf{v}^T\mathbf{w}_i)^2$. Applying Lemma \ref{lma: sub_exp_sum}, for any unit vector $\mathbf{v}$, there is a constant $\eta > 1$ and $t > 0$ such that with probability at least $1-2e^{-t/\eta}$, 
\begin{IEEEeqnarray}{rCl}
|\mathbf{v}^TA\mathbf{v} - \sigma_w^2\sum_{i=1} \lambda_i^{\xi}| &\leq & \sigma_w^2\max \left( \lambda_1^{\xi} \frac{t}{\eta}, \sqrt{\frac{t}{\eta}\sum_{i=1} \left(\lambda_i^{\xi}\right)^2}\right) \leq \frac{1}{\sqrt{\eta}}\sigma_w^2\left(\lambda_1^{\xi} t + \sqrt{t\sum_{i=1} \left(\lambda_i^{\xi}\right)^2}\right). \nonumber
\end{IEEEeqnarray}
Now since $A$ has at most $d$ positive eigenvalues, we let the $d$ dimensional subspace spanned by $A$ be $\mathcal{A}^d$, and let $\mathcal{N}_{\omega}$  be the $\omega$-net of $\mathcal{S}^{d-1}$ with respect to the Euclidean distance, where $\mathcal{S}^{d-1}$ is the unit sphere in $\mathcal{A}^d$. We let $\omega = \frac{1}{4}$, implying that $|\mathcal{N}_{\omega}| \leq 9^d$. Applying the union bound, for every $\mathbf{v} \in \mathcal{N}_{\epsilon}$, we have with probability at least $1-2 e^{-t/\eta}$ \[\left|\mathbf{v}^TA\mathbf{v} - \sum_{i=1} \lambda_i\right| \leq \frac{1}{\sqrt{\eta}}\sigma_w^2\left(\lambda_1^{\xi} (t + d\log 9) + \sqrt{(t+ d \log 9)\sum_{i=1} \left(\lambda_i^{\xi}\right)^2}\right).\] 
Applying the $\epsilon$-net argument, since $\omega = \frac{1}{4}$, for any $\mathbf{v} \in \mathcal{S}^{d-1}$, we have 
\begin{IEEEeqnarray}{rCl}
\left|\mathbf{v}^TA\mathbf{v} - \sum_{i=1} \lambda_i\right| 
 \leq  \frac{b_1}{\sqrt{\eta}}\sigma_w^2\left(\lambda_1^{\xi} (t + d\log 9) + \sqrt{(t+ d \log 9)\sum_{i=1} \left(\lambda_i^{\xi}\right)^2}\right):= \Lambda.\nonumber
\end{IEEEeqnarray}
Thus, with probability $1- 2e^{-t/\eta}$, we have \[\left\|A - \sum_{i=1} \lambda_i I_d \right\|_2 \leq \Lambda.\]
We now further simplify $\Lambda$. Note that when $t \leq d$, $(t + n \log 9) \leq b_2 d$. Hence,
\begin{IEEEeqnarray}{rCl}
\Lambda & \leq &\frac{b_1}{\sqrt{\eta}}\sigma_w^2 \left(b_2\lambda_1^{\xi} d + \sqrt{b_2d \sum_{i=1}^d \left(\lambda_i^{\xi}\right)^2} \right), \nonumber \\
& \leq & \frac{b_1}{\sqrt{\eta}}\sigma_w^2 \left(b_2\lambda_1^{\xi} d + \sqrt{b_2d \lambda_{1}^{\xi}\sum_{i=1}^d \lambda_i^{\xi}} \right),   \nonumber \\
& \leq & \frac{b_1}{\sqrt{\eta}}\sigma_w^2 \left(b_2\lambda_1^{\xi} d + \frac{b_2b_3d\lambda_1^{\xi}}{2} + \frac{1}{2b_3}\sum_{i=1}^d \lambda_i^{\xi} \right) ~~~ (\text{we~use~}\sqrt{xy} \leq \frac{x+y}{2}),\nonumber\\
& = & \frac{1}{\sqrt{\eta}} \sigma_w^2(b_4\lambda_1^{\xi} d + \frac{1}{b_5} \sum_{i=1}^n \lambda_i^{\xi}). \nonumber
\end{IEEEeqnarray}
Therefore, with probability $1- 2 e^{-d/\eta}$, we have:
\begin{IEEEeqnarray}{rCl}
\sigma_w^2\left((1- \frac{1}{b_5\sqrt{\eta}})\sum_{i=1}^d \lambda_i^{\xi} - \frac{b_6}{\sqrt{\eta}}\lambda_1^{\xi}d\right) \leq \mu_{\min}(A) \leq \mu_1(A) \leq \sigma_w^2\left((1+ \frac{1}{b_5\sqrt{\eta}})\sum_{i=1}^d \lambda_i^{\xi} + \frac{b_6}{\sqrt{\eta}}\lambda_1^{\xi}d\right).  \nonumber
\end{IEEEeqnarray}
Using the same proof for $A_k$, we obtain the bound for $\mu_{\min}(A_k)$ and $\mu_1(A_k)$.
\end{proof}

\subsection{Proof of Theorem \ref{theo:var_lower}} \label{sec:lower_proof}
\low*
\begin{proof}
First, recall that we showed the variance can be written as:
\begin{IEEEeqnarray}{rCl}
\mathbf{V}_{R} = \frac{\sigma_0^2}{n}\mathbb{E}_x\left(\textnormal{Tr}\bigg[\mathbf{z}_{x}^T(\frac{1}{n}\mathbf{Z}_{\xi}^T\mathbf{Z}_{\xi})^{\dagger}\mathbf{z}_{x}\bigg]\right). \nonumber
\end{IEEEeqnarray}

By Lemma \ref{lma:zz_ww_diff}, with probability greater than $1-2e^{-s/b_1}-2e^{-d/b_2}$, \[\left\| \frac{1}{n}\mathbf{Z}_{\xi}^T\mathbf{Z}_{\xi} - \mathbf{W}\Sigma_{\xi}\mathbf{W}\right\|_2 = O\left(\sqrt{\frac{d}{n}} d\sigma_w^2\right).\]

By Lemma \ref{lma:w_sigam_w}, with probability greater than $1-e^{-d/\eta}$, \[\mu_1\left(\mathbf{W}\Sigma_{\xi}\mathbf{W}\right) \leq (1+ \frac{1}{b_3\sqrt{\eta}})\sigma_w^2\sum_{i=1}^d \lambda_i^{\xi} + \frac{b_4}{\sqrt{c_2}}\sigma_w^2\lambda_1^{\xi}d \leq b_5d\sigma_w^2.\]

Therefore, it is easy to see that with probability greater than $1-2e^{-s/b_1}-3e^{-d/b_2}$, we have \[\mu_1\left(\frac{1}{n}\mathbf{Z}_{\xi}^T\mathbf{Z}_{\xi} \right) \leq b_5 d\sigma_w^2.\]

Hence, we have 
\begin{IEEEeqnarray}{rCl}
\mathbf{V}_{R} &=& \frac{\sigma_0^2}{n}\mathbb{E}_x\left(\textnormal{Tr}\bigg[\mathbf{z}_{x}^T(\frac{1}{n}\mathbf{Z}_{\xi}^T\mathbf{Z}_{\xi})^{\dagger}\mathbf{z}_{x}\bigg]\right) \nonumber\\
&\geq &  \frac{\sigma_0^2}{n} \frac{1}{b_5d\sigma_w^2}\mathbb{E}_x\left(\textnormal{Tr}\bigg[\mathbf{z}_{x}^T\mathbf{z}_{x}\bigg]\right)\nonumber \\
&\geq & b_6 \frac{\sigma_0^2}{nd\sigma_w^2} \frac{1}{b_7}s\sigma_w^2\mathbb{E}_{x}\left(\|x\|_2^2\right) ~~(\text{Lemma}~\ref{lma:norm_sug})\nonumber \\
&\geq& \frac{1}{b_8}\sigma_0^2\textnormal{Tr}\left(\Sigma\right)\frac{s}{nd}.\nonumber
\end{IEEEeqnarray}
The final result follows by letting $c = \max\{b_1,\dots\}$.

\end{proof}

\section{Proof of Theorem \ref{theo:bias_upper}} \label{sec:bias_proof}
\bias*
\begin{proof}
Recall from Lemma \ref{lma:bia_var} we have $\mathbf{B}_R = \mathbb{E}_{x}\|\mathbf{z}_{x}^T\Pi_{\xi}\beta_*\|^2$, where $\Pi_{\xi} = \left(\mathbf{Z}_{\xi}^T\mathbf{Z}_{\xi}\right)^{\dagger}\mathbf{Z}_{\xi}^T\mathbf{Z}_{\xi} -I_s$.\footnote{Note the $\mathbf{B}_R$ now contains $\Pi_{\xi}$ because our sample covaraites are corrupted with noise $\xi$.} This can be further written as:
\begin{IEEEeqnarray}{rCl}
\mathbf{B}_R &= & \mathbb{E}_{x} \left( \beta_*^T\Pi_{\xi} \mathbf{z}_{x}\mathbf{z}_{x}^T\Pi_{\xi}\beta_*\right), \nonumber\\
&=&\mathbb{E}_{x} \left\{ \beta_*^T\Pi_{\xi} \left(\mathbf{z}_{x}\mathbf{z}_{x}^T- \frac{1}{n}\mathbf{Z}_{\xi}^T\mathbf{Z}_{\xi}\right)\Pi_{\xi}\beta_*\right\},\nonumber\\
& =& \beta_*^T\Pi_{\xi} \bigg\{\mathbb{E}_{x}\left(\mathbf{z}_{x}\mathbf{z}_{x}^T\right) - \frac{1}{n}\sum_{i=1}^n \mathbf{z}_{x_i+\xi_i}\mathbf{z}_{x_i+\xi_i}^T\bigg\}\Pi_{\xi}\beta_*, \label{eqn:bias_eqn}
\end{IEEEeqnarray}
where in the second step, we used the fact that  \[\mathbf{Z}_{\xi}^T\mathbf{Z}_{\xi} \Pi_{\xi} =\mathbf{Z}_{\xi}^T\mathbf{Z}_{\xi}\left(I -(\mathbf{Z}_{\xi}^T\mathbf{Z}_{\xi})^{\dagger}\mathbf{Z}_{\xi}^T\mathbf{Z}_{\xi}\right) = 0.\]
Hence, to be able to find the behavior of $\mathbf{B}_R$, we need to study $\mathbb{E}_{x}\left(\mathbf{z}_{x}\mathbf{z}_{x}^T\right) - \frac{1}{n}\sum_{i=1}^n \mathbf{z}_{x_i+\xi_i}\mathbf{z}_{x_i+\xi_i}^T$. 

We first notice that
\begin{IEEEeqnarray}{rCl}
&&\mathbb{E}_{x}\left(\mathbf{z}_{x}\mathbf{z}_{x}^T\right) - \frac{1}{n}\sum_{i=1}^n \mathbf{z}_{x_i+\xi_i}\mathbf{z}_{x_i+\xi_i}^T \nonumber \\
&=&\mathbb{E}_{x}\left(\mathbf{z}_{x}\mathbf{z}_{x}^T\right) - \frac{1}{n}\sum_{i=1}^n \mathbf{z}_{x_i}\mathbf{z}_{x_i}^T\nonumber\\
&& + \frac{1}{n}\sum_{i=1}^n \mathbf{z}_{x_i}\mathbf{z}_{x_i}^T - \frac{1}{n}\sum_{i=1}^n \mathbf{z}_{x_i+\xi_i}\mathbf{z}_{x_i+\xi_i}^T\nonumber\\
&=& \mathbb{E}_{x}\left(\mathbf{z}_{x}\mathbf{z}_{x}^T\right) - \frac{1}{n}\sum_{i=1}^n \mathbf{z}_{x_i}\mathbf{z}_{x_i}^T \label{eqn:bias_exp}\\
&& + \frac{1}{n}\sum_{i=1}^n \mathbf{z}_{x_i}\mathbf{z}_{x_i}^T - \frac{1}{n}\sum_{i=1}^n \mathbf{z}_{x_i+\xi_i}\mathbf{z}_{x_i}^T \label{eqn:zz_zxi}\\
&& + \frac{1}{n}\sum_{i=1}^n \mathbf{z}_{x_i+\xi_i}\mathbf{z}_{x_i}^T-\frac{1}{n}\sum_{i=1}^n \mathbf{z}_{x_i+\xi_i}\mathbf{z}_{x_i+\xi_i}^T.\label{eqn:zxi_xixi}
\end{IEEEeqnarray}

For Eq.~(\ref{eqn:bias_exp}), we apply Lemma \ref{lma:bern} to obtain that with probability greater than $1-\delta-2e^{-s/b_1}$, we have \begin{eqnarray*}
\left\|\mathbb{E}_{x}\left(\mathbf{z}_{x}\mathbf{z}_{x}^T\right) - \frac{1}{n}\sum_{i=1}^n \mathbf{z}_{x_i}\mathbf{z}_{x_i}^T\right\| \leq b_2\sqrt{\frac{1}{n}\log\frac{s}{\delta}}.
\end{eqnarray*}

For Eq.~(\ref{eqn:zz_zxi}), we have 
\begin{IEEEeqnarray}{rCl}
&&\frac{1}{n}\sum_{i=1}^n \mathbf{z}_{x_i}\mathbf{z}_{x_i}^T - \frac{1}{n}\sum_{i=1}^n \mathbf{z}_{x_i+\xi}\mathbf{z}_{x_i}^T \nonumber\\
&=& \frac{1}{n}\sum_{i=1}^n \left(\mathbf{z}_{x_i}-\mathbf{z}_{x_i+\xi_i}\right)\mathbf{z}_{x_i}^T \nonumber\\
&=& \frac{1}{n} \Delta\mathbf{Z}^T\mathbf{Z}. \nonumber
\end{IEEEeqnarray}
Applying Lemma \ref{lma:deltaz}, we have with probability greater than $1-2e^{-d/b_3}$, \[\left\|\frac{1}{n} \Delta\mathbf{Z}^T\mathbf{Z}\right\| \leq b_4\frac{d^2\sigma_{\xi}}{n}.\]

Eq.~(\ref{eqn:zxi_xixi}) can be written as $\frac{1}{n}\Delta\mathbf{Z}\mathbf{Z}_{\xi}$. We notice that $\mathbf{Z}_{\xi} = \sigma(\mathbf{W}(x_i+\epsilon_i))$, which is similar to $\mathbf{Z} =  \sigma(\mathbf{W}x_i)$, since $x_i + \epsilon_i$ is also a subgaussian random variable by our assumption. Therefore, we apply Lemma \ref{lma:deltaz} similarly and obtain with probability greater than $1-2e^{-d/b_5}$, \[\left\|\frac{1}{n} \Delta\mathbf{Z}^T\mathbf{Z}_{\xi}\right\| \leq b_6\frac{d^2\sigma_{\xi}}{n}.\]

Combining the results together and letting $c = \max\{b_1,\dots\}$, we have with probability greater than $1-\delta-2e^{-s/c}-4e^{-d/c}$, \[\mathbf{B}_R \leq c \left\{\sqrt{\frac{1}{n}\log\frac{s}{\delta}} +\frac{d^2\sigma_{\xi}}{n} \right\}.\]
\end{proof}

\begin{lma}\label{lma:bern}
With probability greater than $1-\delta-2e^{-s/c}$, we have 
\begin{eqnarray*}
\left\|\mathbb{E}_{x}\left(\mathbf{z}_{x}\mathbf{z}_{x}^T\right) - \frac{1}{n}\sum_{i=1}^n \mathbf{z}_{x_i}\mathbf{z}_{x_i}^T\right\| \leq c\sqrt{\frac{1}{n}\log\frac{s}{\delta}},
\end{eqnarray*}
where $\delta \in (0,1)$ and $c$ are some universal constants.
\end{lma}

\begin{proof}
We adopt the matrix concentration inequality in Section \ref{sec:mat_bern}. Define $R_i = \mathbf{z}_{x_i}\mathbf{z}_{x_i}^T $ and $R = \mathbb{E}_{x}\left(\mathbf{z}_{x}\mathbf{z}_{x}^T\right)$ and recall \[\mathbf{z}_{x_i} = [(x_i^T\mathbf{w}_1)\mathbbold{1}_{\{x_i^T\mathbf{w}_1>0 \}}, \dots, (x_i^T\mathbf{w}_s)\mathbbold{1}_{\{x_i^T\mathbf{w}_s >0 \}}]^T = D_{x_i}\mathbf{W}x_i.\] Clearly, we have $\mathbb{E}_{x}(R_i) = R.$ In addition,

\[\|R_i\|_2 = \|\mathbf{z}_{x_i}\mathbf{z}_{x_i}^T\|_2 = \mathbf{z}_{x_i}^T\mathbf{z}_{x_i}.\]
By Lemma \ref{lma:zx_stats}, each entry of $\mathbf{z}_{x_i}$ is i.i.d subgaussian with mean $b_1 \sigma_w\|x_i\|$ and variance $b_2\sigma_w^2\|x_i\|_2^2$. By Lemma \ref{lma:norm_sug}, we have with probability greater than $1- 2e^{-s/b_3}$, 
\begin{eqnarray}
\|R_i\|_2\leq b_4 s \sigma_w^2 \|x_i\|_2^2 = b_4 u_i^T \Sigma u_i, \nonumber
\end{eqnarray}
where we used the assumption that $x_i = \Sigma u_i$. Since $u_i$ is a subgaussian random vector with i.i.d entries, then applying Lemma 36 in \cite{page2019ivanov}, we have with probability greater than $1-2e^{-t}$, 
\begin{eqnarray}
\|R_i\|_2\leq b_4 4t \textnormal{Tr}(\Sigma) \leq b_5t. \label{eqn:Ri_bound}
\end{eqnarray}

$R_i$ is symmetric, so $R_i^TR_i = R_iR_i^T$ and 
\begin{IEEEeqnarray*}{rCl}
R_iR_i^T &=& D_{x_i}\mathbf{W}x_ix_i^T\mathbf{W}^TD_{x_i}D_{x_i}\mathbf{W}x_ix_i^T\mathbf{W}^TD_{x_i}\\
&\preceq & D_{x_i}\mathbf{W}x_ix_i^T\mathbf{W}^T\mathbf{W}x_ix_i^T\mathbf{W}^TD_{x_i}\\
&=& \left\| \mathbf{W}x_i\right\|_2 D_{x_i}\mathbf{W}x_ix_i^T\mathbf{W}^TD_{x_i} \\
& \preceq & b_6s\sigma_w^2\mathbf{W}x_ix_i^T\mathbf{W}^T,
\end{IEEEeqnarray*}
where for the last inequality, we used the fact that $\mathbf{W}x_i$ is a Gaussian random vector with i.i.d entries and Lemma \ref{lma:norm_sug}. We also used that $\mathbf{W}x_ix_i^T\mathbf{W}^T$ is positive semidefinite. We now have \[\mathbb{E}_{x}\left(R_iR_i^T\right) \preceq b_6\mathbf{W} \Sigma\mathbf{W}^T = b_6\sum_{i=1}^d \lambda_i \mathbf{W}_i\mathbf{W}_i^T.\] Recall that we have $d < s$. Applying Lemma \ref{lma:eigen_A_bd}, we have with probability greater than $1-2e^{-s/b_7}-2e^{-d/b_8},$ \[\mathbb{E}_{x}\left(R_iR_i^T\right) \preceq b_6\sigma_w^2  (\sum_{i=1}^d \lambda_i + d\lambda_1)I_s\preceq b_9d\sigma_w^2 I_s \preceq b_9 I_s.\] We are now ready to apply Lemma \ref{lma:matx_con}. With probability greater than $1-2e^{-t}-2e^{-s/b_7}-2e^{-d/b_8}$, we have 
\[P\left(\left\|\mathbb{E}_{x}\left(\mathbf{z}_{x}\mathbf{z}_{x}^T\right) - \frac{1}{n}\sum_{i=1}^n \mathbf{z}_{x_i}\mathbf{z}_{x_i}^T\right\| \geq \epsilon\right) \leq 4s \exp\left(\frac{-n\epsilon^2}{b_9 + 2b_5t\epsilon/3}\right) := \delta.\] 

If we let $t \leq \frac{d}{b_{10}}$ and $ \sqrt{b_9 /n}+ \frac{2b_5t}{3n} \leq  \epsilon  \leq \frac{1}{t}$, and rearrange the above equation, we have with probability greater than $1-\delta-2e^{-d/b_{10}}-2e^{-s/b_7}-2e^{-d/b_8}$,
\begin{eqnarray*}
\left\|\mathbb{E}_{x}\left(\mathbf{z}_{x}\mathbf{z}_{x}^T\right) - \frac{1}{n}\sum_{i=1}^n \mathbf{z}_{x_i}\mathbf{z}_{x_i}^T\right\|\leq b_{11}\sqrt{\frac{1}{n}\log\frac{s}{\delta}}.
\end{eqnarray*}
The final result follows by letting $c = \max\{b_1,\dots\}$
\end{proof}

\begin{lma}\label{lma:deltaz}
Define $\Delta\mathbf{Z} \in \mathbb{R}^{n\times s}$ and $\Delta\mathbf{Z}_{ij}= \mathbf{z}_{x_i,j}-\mathbf{z}_{x_i+\xi_i,j} = \sigma(\mathbf{w}_j^Tx_i)- \sigma(\mathbf{w}_j^T(x_i+\xi_i))$, recall $\mathbf{Z} = [\mathbf{z}_{x_1},\dots,\mathbf{z}_{x_n}]^T$, then we have with probability greater than $1-2e^{-d/c}$, \[\left\|\frac{1}{n} \Delta\mathbf{Z}^T\mathbf{Z}\right\| \leq c\frac{d^2\sigma_{\xi}}{n},\] where $c$ is some universal constant.
\end{lma}
\begin{proof}
It is easy to see that \[\left\|\frac{1}{n} \Delta\mathbf{Z}^T\mathbf{Z} \right\| \leq \frac{1}{n} \|\Delta\mathbf{Z}\|_2 \|\mathbf{Z}\|_2.\]

For $\|\Delta\mathbf{Z}\|_2$, we notice that 
\begin{IEEEeqnarray*}{rCl}
\|\Delta\mathbf{Z}\|_2 &\leq& \sqrt{\textnormal{Tr}(\Delta\mathbf{Z}^T\Delta\mathbf{Z})}\\
& =& \sqrt{\sum_{i=1}^n\sum_{j=1}^s \Delta\mathbf{Z}_{ij}^2}\\
&\leq & \sqrt{\sum_{i=1}^n\sum_{j=1}^s (2\mathbf{w}_j^T\xi_i)^2}\\
&=& 2\sqrt{\textnormal{Tr} \left(\mathbf{W}\Xi^T\Xi\mathbf{W}^T\right)}.
\end{IEEEeqnarray*}
The first inequality is because the operator norm is no larger than the Frobenius norm. The third inequality is a property of the ReLU function where $|\sigma(X+Y)-\sigma(X)| \leq 2|Y|$. Finally we use $\Xi$ to denote the matrix $[\xi_1,\dots,\xi_n]^T$.

We can write $\Xi^T\Xi = \sum_{i=1}^n \xi_i\xi_i^T$. Applying Lemma \ref{lma:eigen_A_bd} and noticing $d < n$, we have with probability greater than $1- 2e^{-d/b_1}$, \[\mu_1(\Xi^T\Xi ) \leq b_2 d\sigma_{\xi}^2.\]

Hence, 
\begin{IEEEeqnarray*}{rCl}
\|\Delta\mathbf{Z}\|_2 &\leq& b_2\sqrt{d\sigma_{\xi}^2\textnormal{Tr} \left(\mathbf{W}\mathbf{W}^T\right)}\\
&=& b_2\sqrt{d\sigma_{\xi}^2\sum_{i=1}^d \|\mathbf{W}_i\|_2^2}.
\end{IEEEeqnarray*}

Note that $\sum_{i=1}^d \|\mathbf{W}_i\|_2^2$ is a weighted sum of $\sigma_w^2$-subexponential random variables with weights all equal to $1$ given in block size of $d$. Hence we apply Lemma \ref{lma: sub_exp_sum}, and obtain that with probability greater than $1-2e^{-d/b_3}$, \[\sum_{i=1}^d \|\mathbf{W}_i\|_2^2 \leq b_4ds\sigma_w^2 = b_4d.\]

Combining the results together, we have with probability greater than $1-4e^{-d/b_5}$, \[\|\Delta\mathbf{Z}\|_2 \leq b_6 d \sigma_{\xi}.\]

Similarly, 
\begin{IEEEeqnarray*}{rCl}
\|\mathbf{Z}\|_2 &\leq& \sqrt{\textnormal{Tr}\left( \mathbf{Z}^T\mathbf{Z}\right)}\\
&\leq & \sqrt{\sum_{i=1}^n\sum_{j=1}^s (\sigma(\mathbf{w}_j^Tx_i))^2}\\
&\leq & \sqrt{\sum_{i=1}^n\sum_{j=1}^s (\mathbf{w}_j^Tx_i)^2}\\
&\leq & \sqrt{\textnormal{Tr}\left(\mathbf{W}X^TX\mathbf{W}^T\right)}\\
&= &\sqrt{\textnormal{Tr}\left(\mathbf{W}\Sigma^{1/2}U^TU\Sigma^{1/2}\mathbf{W}^T\right)},
\end{IEEEeqnarray*}
where we have defined $U = [u_1,\dots,u_n]^T \in \mathbb{R}^{n\times d}$. Similar to the study of $\Xi^T\Xi$, we apply Lemma \ref{lma:eigen_A_bd} to $U^TU$ so that with probability greater than $1-2e^{-d/b_7}$, we have \[\mu_1(U^TU) \leq b_8d.\] We write $\mathbf{W}\Sigma\mathbf{W}^T = \sum_{i=1}^d \lambda_i \mathbf{W}_i\mathbf{W}_i$, then $\textnormal{Tr}(\mathbf{W}\Sigma\mathbf{W}^T) = \sum_{i=1}^d\|\mathbf{W}_i\|_2^2$. $\sum_{i=1}^d\|\mathbf{W}_i\|_2^2$ is the weighted sum of $\sigma_w^2$-subexponential random variable with weights given by $\lambda_i$ in block size of d. Hence applying Lemma \ref{lma: sub_exp_sum}, we have with probability greater than $1- 2e^{-d/b_9}$, \[\sum_{i=1}^d\|\mathbf{W}_i\|_2^2 \leq b_{10}ds\sigma_w^2 \leq b_{10}d.\]

Combining  the results together, we have with probability greater than $1-4e^{-d/{b_{11}}}$, \[\|\mathbf{Z}\|_{2} \leq b_{12}d.\]

Take $c = \max\{b_1,\dots\}$, we have with probability greater than $1-2e^{-d/c}$, \[\left\|\frac{1}{n} \Delta\mathbf{Z}^T\mathbf{Z}\right\| \leq c \frac{d^2\sigma_{\xi}}{n}.\]
\end{proof}

\section{Concentration Inequality for Subgaussian and Subexponential Random Variables}\label{sec:con_sub}
The next lemma, known as the {\em General Hoeffding's inequality}, is from Theorem 2.6.2 \& 2.6.3 in \cite{vershynin2018high} with some light modifications.
\begin{lma}\label{lma:gen_hoe}
Let $x_{1}, \ldots, x_{n}$ be independent, zero  mean, $\sigma_{x_i}^2$-subgaussian random variables, and let $c,c' >0$ be constants. Then, for every $t \geq 0,$ we have
\[\mathbb{P}\left\{\left|\sum_{i=1}^{n} x_{i}\right| \geq t\right\} \leq 2 \exp \left(-\frac{c t^{2}}{\sum_{i=1}^{n}\sigma_{x_i}^2}\right).\] Equivalently, with probability greater than $1- 2e^{-t}$, we have 
\[\left|\sum_{i=1}^{n} x_{i}\right| \leq c\sqrt{t\sum_{i=1}^n\sigma_{x_i}^2}.\]
In addition, if $\mathbf{a} = [a_1,\dots,a_n]^T \in \mathbb{R}^n$, and letting $\sigma_x^2 = \max_{i \in [n]}\{\sigma_{x_i}^2\}$, then for every $t >0$, we have, \[\mathbb{P}\left\{\left|\sum_{i=1}^{n} a_ix_{i}\right| \geq t\right\} \leq 2 \exp \left(-\frac{c t^{2}}{\sigma_{x}^2\|\mathbf{a}\|_2^2}\right).\] Equivalently, with probability greater than $1- 2e^{-t}$, we have \[\left|\sum_{i=1}^{n}a_i x_{i}\right| \leq c\sqrt{t\sigma_{x}^2\|\mathbf{a}\|_2^2}.\]
\end{lma}

Below is the concentration result for subexponential random varaible from Corollary S.6 in \cite{bartlett2020benign}.

\begin{lma}\label{lma: sub_exp_sum}
Let $\{\lambda_i\}_{i=1}^{\infty}$ be a sequence of non-increasing and non-negative numbers such that $\sum_{i=1}^{\infty} \lambda_i < \infty$. In addition, suppose we have a sequence of i.i.d centered, $\sigma$-subexponential random variables $\{\xi_i\}_{i=1}^{\infty}$. Then there is a universal constant $a$ such that for probability greater than $1-2e^{-t}$, $t >0$,
\[\left|\sum_{i} \lambda_i\xi_i \right| \leq a\sigma \max\left(\lambda_1 t, \sqrt{t\sum_i\lambda_i^2} \right).\]
\end{lma}

The following lemma gives the lower and upper bounds for the norm of a subgaussian random vector. 
\begin{lma}\label{lma:norm_sug}
Let $\mathbf{u} \in \mathbb{R}^n$ be a random vector with each coordinate being an i.i.d. mean $\mu$, variance $\sigma^2$, subgaussian random variable. Then with probability greater than $1- 2e^{-t}$, for a universal constant $a_1$,  \[n\left(\mu^2 + \sigma^2\right) - a_1\sigma^2 \left(t+ \sqrt{nt}\right) \leq \|\mathbf{u}\|^2 \leq n\left(\mu^2 + \sigma^2\right) + a_1\sigma^2 \left(t+ \sqrt{nt}\right).\] In particular, if $t <\frac{n}{a_2}$ for some sufficiently large constant $a_2$ and a universal constant $a_3$ , we have \[\frac{1}{a_3} n(\mu^2+\sigma^2)\leq \|\mathbf{u}\|^2 \leq  a_3 n(\mu^2+\sigma^2).\] 
\end{lma}

\begin{proof}
Notice that $\|\mathbf{u}\|^2 = \sum_{i=1}^n u_i^2$. Let $\xi_i = u_i^2 - (\mu^2+\sigma^2)$, then we can see that $\xi_i$ is $\sigma^2$-subexponential with mean $0$.  Applying Lemma \ref{lma: sub_exp_sum}, we have with probability greater than $1- 2e^{-t}$,\[\left|\sum_{i=1}^n \xi_i \right| \leq b_1\sigma^2\max\left(t,\sqrt{tn}\right).\] This implies that \[\left|\sum_{i=1}^nu_i^2 - n(\mu^2+\sigma^2)\right| \leq b_1\sigma^2(t+\sqrt{nt}),\] which further implies the first argument. For the second one, we let $t \leq \frac{n}{b_2}$, rearranging the terms to obtain \[n(\mu^2+\sigma^2) - b_1n\sigma^2(\frac{1}{b_2}+\sqrt{1/b_2}) \leq \sum_{i=1}^nu_i^2 \leq n(\mu^2+\sigma^2) + b_1n\sigma^2(\frac{1}{b_2}+\sqrt{1/b_2}).\] The second argument follows easily from the above equation. Note that for the lower bound, we require choosing sufficiently large $b_2$ such that $b_1/\sqrt{b_2} < 1$ to ensure that the lower bound is positive.
\end{proof}

Lemma \ref{lma:e-net} is from \cite[Lemma S.8]{bartlett2020benign}.
\begin{lma} \label{lma:e-net} $(\epsilon \text{-net argument}) $  
Let $A \in \mathbb{R}^{n \times n}$ be a symmetric matrix, and $\mathcal{N}_{\epsilon}$ be an $\epsilon$-net on the unit sphere $\mathcal{S}^{n-1}$ with $\epsilon<\frac{1}{2}$. Then we have \[\|A\| \leq(1-\epsilon)^{-2} \max _{\mathbf{a} \in \mathcal{N}_{\epsilon}}\left|\mathbf{a}^{T} A \mathbf{a}\right|.\]
\end{lma}

\section{Concentration Inequality for a Subgaussian Matrix}\label{sec:eigen_A}
The following lemma provides the upper and lower bounds for the eigenspectrum of a subgaussian random matrix. The proof of this lemma closely follows \cite{bartlett2020benign}.
\begin{lma}\label{lma:eigen_A_bd}
Let $A = \sum_{i=1}^n \lambda_i \mathbf{w}_i \mathbf{w}_i^T$, where $\mathbf{w}_i \in \mathbb{R}^{s}$ is a random vector with each entry being i.i.d., mean $0$, variance $1$ and $\sigma_w^2$-subgaussian random variable. Then if we let $N = \min\{s,n\}$, there is a universal constant $a_1$ such that with probability at least $1-2e^{-t}$, we have 
\begin{IEEEeqnarray}{rCl}
\sum_{i=1}^{N} \lambda_{i}-\Lambda \leq \mu_{N}(A) \leq \mu_{1}(A) \leq \sum_{i=1}^{N} \lambda_{i}+\Lambda, \label{eigen_concen_1}
\end{IEEEeqnarray}
where
\[\Lambda=a_1\left(\lambda_{1}(t+N \log 9)+\sqrt{(t+N \log 9) \sum_{i=1}^{N} \lambda_{i}^{2}}\right).\]
Further, there is a universal constant $a_2$ such that with probability at least $1-2e^{-N/a_2}$,
\begin{IEEEeqnarray}{rCl}
\frac{1}{a_2}\sum_{i=1}^{N} \lambda_{i}-&a_2& \lambda_{1} N \leq \mu_{N}(A) \leq \mu_{1}(A) \leq a_2\sum_{i=1}^{N} \lambda_{i}+ a_2\lambda_{1} N.  \label{eigen_concen_2}
\end{IEEEeqnarray}
In addition, with the same probability bound, we have 
\begin{IEEEeqnarray}{rCl}
\frac{1}{a_2}\sum_{i> k}^{N} &\lambda_{i}-&a_2  \lambda_{k+1} N \leq \mu_{N}(A_k) \leq \mu_{1}(A_k) \leq a_2\sum_{i>k}^{N} \lambda_{i}+ a_2\lambda_{k+1}N.  \label{eigen_concen_3}
\end{IEEEeqnarray}
\end{lma}
\begin{proof}
We first assume $s > n$, i.e., $N = n$. For any unit vector $\mathbf{v} \in \mathbb{R}^s$, we have $\mathbf{v}^T\mathbf{w}_i$ is $\sigma_w^2$-subgaussian, implying that $\mathbf{v}^T\mathbf{u}_i\mathbf{w}_i^T\mathbf{v}-1 = (\mathbf{v}^T\mathbf{w}_i)^2-1$ is centered and $\sigma_w^2$ subexponential. By Lemma \ref{lma: sub_exp_sum}, for any unit vector $\mathbf{v}$, there is a universal constant $b_1$, such that with probability at least $1-2 e^{-t}$, \[|\mathbf{v}^TA\mathbf{v} - \sum_{i=1}^n \lambda_i| \leq b_1 \left(\lambda_1 t +\sqrt{t \sum_{i=1}^n \lambda_{i}^{2}}\right).\] Since $A$ has at most $n$ non-negative eigenvalues, we let the $n$ dimensional subspace spanned by $A$ be $\mathcal{A}^n$, and let $\mathcal{N}_{\omega}$  be the $\omega$-net of $\mathcal{S}^{n-1}$ with respect to the Euclidean distance, where $\mathcal{S}^{n-1}$ is the unit sphere in $\mathcal{A}^n$. We let $\omega = \frac{1}{4}$, implying that $|\mathcal{N}_{\omega}| \leq 9^n$. Applying the union bound, for every $\mathbf{v} \in \mathcal{N}_{\epsilon}$, we have with probability at least $1-2e^{-t}$, 
\begin{IEEEeqnarray}{rCl}
\left|\mathbf{v}^TA\mathbf{v} - \sum_{i=1}^n \lambda_i\right| \leq b_1&\Bigg(&\lambda_{1}(t+n \log 9)+\sqrt{(t+n \log 9) \sum_{i=1}^{n} \lambda_{i}^{2}}\Bigg).\nonumber
\end{IEEEeqnarray}

Applying the $\epsilon$-net argument (Lemma \ref{lma:e-net}), and since $\omega = \frac{1}{4}$, then for any $\mathbf{v} \in \mathcal{S}^{n-1}$, 
\begin{IEEEeqnarray}{rCl}
|\mathbf{v}^TA\mathbf{v} - \sum_{i=1}^n \lambda_i| &\leq& b_2\bigg(\lambda_{1}(t+n \log 9)+\sqrt{(t+n \log 9) \sum_{i=1}^{n} \lambda_{i}^{2}}\bigg)
:= \Lambda.\nonumber
\end{IEEEeqnarray}

Thus, with probability $1- 2e^{-t}$, \[\|A - \sum_{i=1}^n \lambda_i I_n \| \leq \Lambda.\]
We now further simplify $\Lambda$. Note that when $t \leq \frac{n}{b_3}$, $(t+n \log 9) \leq b_4 n$. Hence,
\begin{IEEEeqnarray}{rCl}
\Lambda & \leq & b_5\lambda_1n + \sqrt{b_6n \lambda_1 \sum_{i=1}^n \lambda_i} \nonumber \\
& \leq & b_5\lambda_1n + \frac{1}{2}b_6b_7\lambda_1n + \frac{1}{2b_7}\sum_{i=1}^n \lambda_i.\nonumber
\end{IEEEeqnarray}
Combining this with Eq.(\ref{eigen_concen_1}) yields Eq.(\ref{eigen_concen_2}). Using the same proof with $A_{k}$ replacing $A$, we obtain Eq.(\ref{eigen_concen_3}). The case for $s < n$ follows the same procedure, hence we omit its proof here.
\end{proof}

\section{Matrix Bernstein Inequality}\label{sec:mat_bern}
The following matrix Bernstein inequality is a result from Lemma 27 in \cite{avron2017random}, which is a restatement of Corollary 7.3.3 in \cite{tropp2015introduction} with some fix in the typos.
\begin{lma}\cite[Corollary 7.3.3, Bernstein Inequality ]{tropp2015introduction}\label{lma:matx_con}
Let $\mathbf{R}$ be a fixed $d_1 \times d_2$ matrix over the set of complex/real numbers. Suppose that $\{\mathbf{R}_1,\cdots,\mathbf{R}_n\}$ are i.i.d samples of $d_1 \times d_2$ matrices such that \[\mathbb{E}[\mathbf{R}_i] = \mathbf{R} \qquad \text{and} \qquad \|\mathbf{R}_i\|_2 \leq L,\]
where $L>0$ is a constant independent of the sample.
Furthermore, let $\mathbf{M}_1, \mathbf{M}_2$ be semidefinite upper bounds for the matrix-valued variances 
\begin{align*}
\begin{aligned}
&  \mathbb{E}[\mathbf{R}_i\mathbf{R}_i^{T}] \preceq \mathbf{M}_1 & \\
&  \mathbb{E}[\mathbf{R}_i^{T}\mathbf{R}_i]\preceq \mathbf{M}_2. &
\end{aligned}
\end{align*}
Let $m = \max(\|\mathbf{M}_1\|_2,\|\mathbf{M}_2\|_2)$ and $d =\frac{\text{Tr}(\mathbf{M}_1)+ \text{Tr}(\mathbf{M}_2)}{m}.$ 
Then, for $\epsilon \geq \sqrt{m/n}+2L/3n$, we can bound \[\bar{\mathbf{R}}_n = \frac{1}{n}\sum_{i=1}^{n}\mathbf{R}_i\] around its mean using the concentration inequality \[P(\|\bar{\mathbf{R}}_n - \mathbf{R}\|_2 \geq \epsilon) \leq 4d\exp\Bigg(\frac{-n\epsilon^2/2}{m+2L\epsilon/3}\Bigg).\]
\end{lma}

\section{A Statistical Property of the ReLU Activation Function}
\begin{lma}\label{lma:relu_stats}
Let $\mathsf{w} \sim \mathcal{N}(0,\sigma^2)$, and define the random variable $\mathsf{x} = \mathsf{w}\mathbbold{1}_{\{\mathsf{w}>0\}}$, where $\mathbbold{1}_{A}$ is the indicator function for event $A$. Then $\mathsf{x}$ has the following cumulative distribution function:
\[F(\mathsf{x} \leq x)=\left\{\begin{array}{ll}
0, & x < 0 ; \\
\frac{1}{2}, &x=0;\\
\frac{1}{2} + \int_{0}^x \frac{1}{\sqrt{2\pi \sigma^2}} e^{-\frac{t^2}{2\sigma^2}} dt & x>0. 
\end{array}\quad\right.\] Furthermore, $\mathsf{x}$ has mean and variance as: \[\mathbb{E}(\mathsf{x}) = \sqrt{\frac{\sigma^2}{2\pi}}, ~~~~~\text{Var}(\mathsf{x}) = \frac{\sigma^2}{2}(1- \frac{1}{\pi}).\] Moreover, $\mathsf{x}$ is $\sigma^2$-subgaussian in the sense that: \[P(|\mathsf{x}| \geq t) \leq 2 \exp\left(-\frac{t^2}{2\sigma^2}\right).\]
\end{lma}
\begin{proof}
We first investigate the CDF of $\mathsf{x}$. It is easy to see that $\mathsf{x}$ is non-negative. In addition, $\mathsf{x} =0$ if and only if $\mathsf{w} \leq 0$, hence $\mathsf{x} = 0$ with probability $\frac{1}{2}$. Finally, if $x > 0$, then we have
\begin{IEEEeqnarray}{rCl}
P(\mathsf{x} \leq x) &=& P(\mathsf{w}\leq 0) + P(0 < \mathsf{w} \leq x),\nonumber\\
&=& \frac{1}{2} + \int_{0}^x \frac{1}{\sqrt{2\pi \sigma^2}} e^{-\frac{t^2}{2\sigma^2}} dt.\nonumber
\end{IEEEeqnarray}
We now compute its mean using the CDF.
\begin{IEEEeqnarray}{rCl}
\mathbb{E}(\mathsf{x}) &=& \frac{1}{2} \times 0 + \int_{0}^{\infty} t\frac{1}{\sqrt{2\pi \sigma^2}} e^{-\frac{t^2}{2\sigma^2}} dt,\nonumber\\
&=& \int_{0}^{\infty} t\frac{1}{\sqrt{2\pi \sigma^2}} e^{-\frac{t^2}{2\sigma^2}} dt,\nonumber\\
& = & \frac{\sigma}{\sqrt{2\pi}}\int_{0}^{\infty}e^{-\frac{t^2}{2\sigma^2}} d\left(\frac{t^2}{2\sigma^2}\right),\nonumber\\
& =& \sqrt{\frac{\sigma^2}{2\pi}}. \nonumber
\end{IEEEeqnarray}
Similarly, we have 
\begin{IEEEeqnarray}{rCl}
\mathbb{E}(\mathsf{x}^2) &=& \frac{1}{2} \times 0 + \int_{0}^{\infty} t^2\frac{1}{\sqrt{2\pi \sigma^2}} e^{-\frac{t^2}{2\sigma^2}} dt,\nonumber\\
&=& \int_{0}^{\infty} t^2\frac{1}{\sqrt{2\pi \sigma^2}} e^{-\frac{t^2}{2\sigma^2}} dt,\nonumber\\
&=& \frac{1}{2} \int_{-\infty}^{\infty} t^2\frac{1}{\sqrt{2\pi \sigma^2}} e^{-\frac{t^2}{2\sigma^2}} dt,\nonumber \\
& =&\frac{1}{2} \mathbb{E}(\mathsf{w}^2)= \frac{1}{2}\sigma^2, \nonumber
\end{IEEEeqnarray}
where for the third equality, we used the symmetry of the integration function. The variance can now easily be  derived. We can also upper bound its moment generating function as:
\begin{IEEEeqnarray}{rCl}
\mathbb{E}(\exp(\lambda \mathsf{x})) &=& \frac{1}{2} + \int_{0}^{\infty} e^{\lambda t} \frac{1}{\sqrt{2\pi \sigma^2}} e^{-\frac{t^2}{2\sigma^2}} dt,\nonumber\\
&\leq& 1 + \int_{-\infty}^{\infty} e^{\lambda t} \frac{1}{\sqrt{2\pi \sigma^2}} e^{-\frac{t^2}{2\sigma^2}} dt,\nonumber\\
& \leq &  e^{\frac{\sigma^2}{2}\lambda^2} + e^{\frac{\sigma^2}{2}\lambda^2}, \nonumber\\
&=& 2 e^{\frac{\sigma^2}{2}\lambda^2}. \nonumber
\end{IEEEeqnarray}
Finally, we have 
\begin{IEEEeqnarray}{rCl}
P(|\mathsf{x}| \geq t) &=& P(\mathsf{x} \geq t) = P(e^{s\mathsf{x}} \geq e^{st}), \nonumber \\
& \leq & \frac{\mathbb{E}(e^{s\mathsf{x}})}{e^{st}} \nonumber ~~~~~\text{(Markov~Inequality)},\\
& \leq & 2 \exp\left(\frac{\sigma^2s^2}{2}-st\right). \nonumber
\end{IEEEeqnarray}
If we let $s = \frac{t}{\sigma^2}$, then we get that \[P(|\mathsf{x}| \geq t) \leq 2 \exp\left(-\frac{t^2}{2\sigma^2}\right) .\] 
\end{proof}

\begin{lma}\label{lma:zx_stats}
Let $\mathbf{W}\in \mathbb{R}^{s\times d}$ be a Gaussian random matrix with each entry i.i.d $\sim \mathcal{N}(0,\sigma_w^2)$. Recall $\mathbf{z}_{x} = \sigma(\mathbf{W}x) \in \mathbb{R}^s$. Then each entry of $\mathbf{z}_{x}$ is i.i.d subgaussian with mean $c_1 \sigma_w\|x\|_2$ and variance $c_2 \sigma_w^2\|x\|_2^2$ where $c_1,c_2$ are some universal constants.
\end{lma}

\begin{proof}
Notice that $\mathbf{z}_{x,j} = \sigma(\mathbf{w}_j^Tx) $ where $\mathbf{w}_j$ is the $j$-th row of $\mathbf{W}$. It is easy to see that $\mathbf{w}_j^Tx \sim \mathcal{N}(0,\sigma_w^2\|x\|_2^2)$. Applying Lemma \ref{lma:relu_stats} yields the mean and variance. Independence follows from the fact that each of the rows of $\mathbf{W}$ are independent.
\end{proof}

\end{document}